\documentclass[11pt]{article}
\usepackage[letterpaper,top=1in,bottom=1in,left=1in,right=1in,marginparwidth=1.75cm]{geometry}

\usepackage{amsfonts}
\usepackage{mathtools,amssymb, amsthm}

\usepackage[most]{tcolorbox}
\usepackage{graphicx}
\usepackage[colorlinks=true, allcolors=blue]{hyperref}
\usepackage{cleveref}
\usepackage{algorithm}
\usepackage{algpseudocode}
\usepackage{todonotes}
\usepackage{lscape}
\usepackage{booktabs}
\usepackage{authblk}
\usepackage{geometry}
\usepackage{pdflscape}
\usepackage{makecell}
\usepackage{utfsym}

\newcommand\keywordsname{Key words}
\usepackage{hyperref}
\usepackage{caption}
\usepackage{subcaption}

\newcommand{\de}{\,\text{d}}
\newcommand{\R}{\mathbb{R}}

\newcommand{\Ex}{\mathbb{E}}
\newcommand{\bfc}{\mathbf{c}}

\newcommand{\bfu}{\mathbf{u}}

\newcommand{\bfZ}{\mathbf{Z}}
\newcommand{\bbQ}{\mathbb{Q}}

\newtheorem{theorem}{Theorem}[section]

\newtheorem{example}{Example}
\newtheorem{remark}{Remark}[section]

\title{\sc Probabilistic operator learning: \\ generative modeling and uncertainty quantification 
for foundation models of differential equations}

\author[a]{Benjamin J. Zhang$^\ast$}
\author[b]{Siting Liu$^\ast$}
\author[c]{Stanley J. Osher}
\author[d]{Markos A. Katsoulakis}

\affil[a]{Division of Applied Mathematics, Brown University}
\affil[b]{Department of Mathematics, University of California, Riverside}
\affil[c]{Department of Mathematics, University of California, Los Angeles}
\affil[d]{Department of Mathematics and Statistics, University of Massachusetts Amherst}

\makeatletter
\def\blfootnote{\gdef\@thefnmark{}\@footnotetext}
\makeatother
\newenvironment{@abssec}[1]{
     \if@twocolumn
       \section*{#1}
     \else
       \vspace{.05in}\footnotesize
       \parindent .0in
         {\upshape\bfseries #1. }\ignorespaces
     \fi}
     {\if@twocolumn\else\par\vspace{.1in}\fi}
\newenvironment{keywords}{\begin{@abssec}{\keywordsname}}{\end{@abssec}}

\begin{document}

\maketitle

\begin{abstract}
In-context operator networks (ICON) are a class of operator learning methods based on the novel architectures of foundation models. Trained on a diverse set of datasets of initial and boundary conditions paired with corresponding solutions to ordinary and partial differential equations (ODEs and PDEs), ICON learns to map example condition-solution pairs of a given differential equation to an approximation of its solution operator. Here, we present a probabilistic framework that reveals ICON as implicitly performing Bayesian inference, where it computes the mean of the posterior predictive distribution over solution operators conditioned on the provided context, i.e., example condition-solution pairs. The formalism of random differential equations provides the probabilistic framework for describing the tasks ICON accomplishes while also providing a basis for understanding other multi-operator learning methods. This probabilistic perspective provides a basis for extending ICON to \emph{generative} settings, where one can sample from the posterior predictive distribution of solution operators. The generative formulation of ICON (GenICON) captures the underlying uncertainty in the solution operator, which enables principled uncertainty quantification in the solution predictions in operator learning.

 \end{abstract}

\blfootnote{$^\ast$Equal Contribution. Emails: \texttt{benjamin\_zhang@brown.edu}, \texttt{sitingl@ucr.edu}, \texttt{sjo@math.ucla.edu}, \texttt{markos@umass.edu}. M. Katsoulakis and B. Zhang are partially funded by AFOSR grant FA9550-21-1-0354. M.K. is partially funded by  NSF DMS-2307115.}

\normalsize

\section{Introduction}
\label{sec:intro}

Operator learning is a class of machine learning methods for approximating solutions to ordinary and partial differential equations. Methods such as Fourier neural operators \cite{li2021fourier}, DeepONets \cite{lu2019deeponet}, in-context operator networks (ICON) \cite{yang2023context,yang2023fine,yang2024pde}, and related transformer-based operator learning methods \cite{herde2024poseidon}
represent many scientific machine learning approaches for solving differential equations. These efforts are part of a broader goal of developing \emph{foundation models} or \emph{multi-operator learning} for differential equations, i.e., single pre-trained models that learn from a collection of many different types of differential equations that can then be fine-tuned to efficiently solve other differential equations \cite{sun2024towards,herde2024poseidon,subramanian2024towards,liu2024prose,liu2025bcat, bodnar2025foundation,hao2024dpot,li2021fourier,mccabe2023multiple}. 

Many operator learning methods, including FNOs and DeepONets, are trained using \emph{supervised learning}, where a mapping is learned from given input data that contains model parameters along with their initial and boundary value data to given output data that consist of solutions \cite{kovachki2024operator}. In contrast, ICON is trained to \emph{learn context}, where the model is not explicitly trained on parameters but instead relies only on example initial-boundary value data and corresponding solution pairs. The ICON architecture is based on transformers, the neural network architecture that underpins large languages models \cite{radford2019language,brown2020language}, which have shown to be capable of solving multiple types of tasks given a particular context defined by example `condition-solution' pairs. ICON is an example of a foundation model for differential equations as it is pre-trained on a variety of ODE and PDE problems, and then fine-tuned to learn specific solution operators given only example pairs of initial-boundary value data and solutions. Our probabilistic formulation of ICON provides a model for \emph{understanding} the mathematics of fine-tuning of foundation models for differential equations. This ability to fine-tune ICON given new example pairs to learn solution operators of unseen differential equations is why \cite{yang2023context,yang2023fine} describes it as an \emph{operator learner} rather than an \emph{operator approximator}: it adapts on-the-fly to new contexts, learning the correct operator without being given explicit information about the underlying model parameters. 



To understand ICON, operator learning, and foundation models of differential equations more generally, we present the notion of \emph{probabilistic} operator learning, which requires modeling the \emph{data generating process} that produces the training datasets. For operator learning, we model this data generating process through the formalism of \emph{random differential equations} (RDEs) \cite{arnold1995random}, upon which we may understand the connections of operator learning and ICON to conditional expectations, least-squares estimators, and Bayesian inference. Random differential equations provide a natural mathematical framework for understanding not just ICON, but many foundation models trained on data from many different differential equations and initial and boundary data. The probabilistic setting also provides a rigorous way of describing confidence in model predictions by characterizing model \emph{uncertainty}. 

ICON learns a function that maps example pairs of input-output data of a differential equation along with a new unseen input, and maps it to the predicted output. This is in contrast to traditional operator learning methods in which a function maps inputs and operator parameters to solution outputs. Our primary contribution is showing that ICON directly approximates the expected value of a posterior predictive distribution, bypassing an explicit characterization of a Bayesian posterior distribution, meaning that ICON performs Bayesian inference implicitly. In other words, ICON infers the appropriate operator parameters given the example pairs and then applies the learned operator to the new inputs to predict the corresponding output in a single step. In-context learning more broadly has been understood to approximate the posterior predictive distribution \cite{xie2022an}, although this interpretation has only been studied in the context of large language models. 

With the knowledge that ICON approximates the \emph{mean} of the posterior predictive distribution, our second main contribution is adapting ICON to \emph{generate} from the posterior predictive distribution. In the context of ICON, the posterior predictive characterizes the distribution of solutions to the differential equation conditioned on example data, meaning that its statistics describes the uncertainty of the predictions as well. Tools from \emph{conditional generative modeling} are needed to generate from the posterior predictive, and we show that being able to produce samples from it provides a way to quantify uncertainty in ICON's predictions. Generating from the posterior predictive distribution is an example of \emph{conditional generative modeling}. When training ICON, or foundation models for differential equations more broadly, we are given input-output pairs of conditions and solutions which come from some joint probability measure. Conditional generative modeling takes input-output pairs to characterize the distribution of outputs given unseen inputs. ICON takes this a step further where it can characterize the distribution of solutions given new conditions and input-output data that inform the model parameters.

We show being able to generate samples from the posterior predictive naturally provides a way to quantify uncertainty in ICON's predictions. As ICON performs Bayesian inference through only joint samples from the inputs and outputs, it is performing a particular version of \emph{likelihood-free inference} \cite{Cranmer2019SBI}. The Bayesian perspective of ICON can provide further insights even for deterministic versions of ICON. In particular, through numerical examples, we show that the Bayesian perspective can inform how many example pairs are necessary in the ICON model, and how ICON can be robust to non-identifiability in differential equations. Moreover, we show ICON is naturally suited to solve Bayesian inference problems involving differential equations without additional computational machinery. 

The paper is organized as follows. In Section~\ref{sec:rdes} we introduce the random differential equations formalism for general ODEs and PDEs, with an emphasis on the probability measures induced by RDEs. The study of the probability measures of RDEs is both natural and crucial for understanding operator learning and extending it to conditional generative modeling.  In Section~\ref{sec:probicon} we present the training objective of in-context operator networks within the RDEs framework; in particular, we place ICON in the context of posterior and posterior predictive distributions, as well as its connections to likelihood-free inference. With the connection between the posterior predictive distribution and ICON established, Section~\ref{sec:genicon} presents a \emph{generative} version of ICON, combining conditional generative modeling with in-context operator learning to produce samples from the posterior predictive distribution, and show that it intrinsically provides uncertainty quantification for the predictions made by ICON. We illustrate computational insights that can be gained through the probabilistic and generative perspective of ICON in Section~\ref{sec:numerical}.

\subsection{Contributions}
We organize the rest of the paper according to our contributions. 
\begin{itemize}
    \item We present the formalism of random differential equations (RDEs) to understand operator learning methods in a probabilistic setting in Section~\ref{sec:rdes}. RDEs are a natural mathematical framework for understanding all foundation models for differential equations trained on data from many different types of ODEs and PDEs and initial and boundary data.  
    

    \item Using the RDE formalism, we explain in-context operator learning as implicitly performing Bayesian inference, in which ICON approximates the mean of the posterior predictive distribution. In Section~\ref{sec:probicon}, we describe in-context learning in probability spaces with Banach and Hilbert space valued random variables. Since solutions of differential equations reside in function spaces, it is natural --- and indeed, necessary --- to adopt this setting as model predictions will also reside in these spaces. This \emph{functional analytic} framework provides the foundation for understanding when least-squares predictors coincide with conditional expectations in operator learning--based approaches to differential equations.

    \item We present a generative form of in-context operator learning that is able to generate from the posterior predictive distribution. This is a form of uncertainty quantification, which describes the level of confidence ICON has in prediction. Building on the abstract formulation in Section~\ref{sec:probicon}, we establish the existence and properties of a generative ICON method in Section~\ref{sec:genicon}, as well as its ability to perform uncertainty quantification. Numerical examples that demonstrate our approach are discussed in Section~\ref{sec:numerical}. 
\end{itemize}

\section{Random differential equations formalism for probabilistic operator learning and foundation models of differential equations} \label{sec:rdes}

As we discussed in Section~\ref{sec:intro}, probabilistic operator learning requires us to define the data generating process that produces the training data for operator learning. Random differential equations provides the probabilistic model that generates this training data, and provides the proper formalism for which many operator learning methods can be understood. In particular, we highlight the probability measures that are induced by random ODEs, boundary value problems, and time-dependent PDEs. Solutions to differential equations are elements of function spaces, i.e., Banach and Hilbert spaces, and therefore, the probability measures we study are defined on Banach and Hilbert spaces.

\subsection{Random differential equations}

Fix a probability space $(\Omega, \mathcal{F},\mathbb{P})$ and let $V$ and $W$ be separable Banach spaces. Let $u(t,\omega)$ be a $V$--valued stochastic process over time period $t \in [0,T]$ evolving according to random evolution process
\begin{align}\label{eq:generalpde}
\begin{cases}
        \frac{\partial}{\partial t} u(t,\omega) + \mathcal{A}[t,u(t,\omega),\omega] =F(t,\omega) \\
    \mathcal{B}[u(t,\omega),\omega] = G(\omega) \\
    u(0,\omega) = u_0(\omega),
    \end{cases}
\end{align}
where $\mathcal{A}: [0,T] \times D(\mathcal{A}) \times \Omega \to V$ is a random operator that governs the system's dynamics and $D(\mathcal{A}) \subset V$ is its domain, $F(t,\cdot): \Omega \to  V$ is a random source term (a $V$-valued random variable),  $\mathcal{B}: D(\mathcal{A}) \times \Omega \to W$ describes the system's boundary conditions $G:\Omega \to W$ is random boundary data, and $u_0:\Omega \to V$ is random initial data. For a realization $\omega \in \Omega$, the system generally describes a single initial-boundary value PDE problem. The randomness here serves mainly to model the ensemble of operators, initial and boundary data, and source terms.  Here, we do not consider randomness caused by stochastic forcing, and therefore do not require a filtration. A simple example to consider is when the randomness only arises from the variability of parameters of a fixed family of random differential equations. 

For a fixed $\omega$, the operator $\mathcal{A}[\cdot,\cdot,\omega]$ can be understood as the generator of a time-inhomogeneous nonlinear semigroup. The solution $u(t,\omega)$ is deterministically coupled to the operators $\mathcal{A}[\cdot,\cdot,\omega], \mathcal{B}[\cdot,\omega]$, initial data $u_0(\omega)$, source term $F(t,\omega)$, and boundary value data $G(\omega)$. The corresponding probability measure of $u(t,\omega)$ encodes the relationship between these random variables. To elucidate the random differential equations framework for operator learning, we give three specific examples in the following sections.

\subsection{Random ODEs} \label{sec:randomode}
For a general random ODE taking values in $V = \R^d$, we have $u(t,\omega)$ to be a $\R^d$-valued stochastic process. A random ODE has no boundary operator conditions, and the random evolution operator $\mathcal{A}$ instantiates as a random vector field $\mathcal{A}[t,u(t,\omega),\omega] = -a(t,u(t,\omega),\omega)$,
\begin{align}\label{eq:randomode}
\begin{cases}
    \frac{d}{dt}u(t,\omega) = a(t,u(t,\omega),\omega)  \\
    u(0,\omega) = u_0(\omega)
\end{cases}
\end{align}
for $t\in [0,T]$. This random ODE can be understood as a collection of ODEs, indexed by $\omega$, characterized by the random vector field $a: [0,T] \times \R^d \times \Omega \to \R^d$. For a fixed $\omega$, \eqref{eq:randomode} is simply an initial-value problem where $a(\cdot,\cdot,\omega)$ is an element of $\mathcal{C}([0,T]\times\R^d; \R^d)$, the space of continuous vector fields taking value in $\R^d$, $u_0(\omega)\in \R^d$ is an initial condition, and the solution $u(\cdot,\omega)$ is an element of a function space, in this paper we assume the space to be ${L}^2(0,T)$. 

The \emph{joint random variable} $(a(\cdot,\cdot,\omega),u_0(\omega),u(\cdot,\omega))$ induces a pathspace probability measure $\bbQ_{a,u_0,u}(\cdot)$ defined on the function spaces to which $u(\cdot,\omega)$ and $a(\cdot,\cdot,\omega)$ belong. The full pathspace measure is neither insightful nor practical to study directly as training data, in practice, is provided through discrete representations. ODEs are first discretized in time and then solved numerically; a more useful measure to study is a marginal distribution $\mathbb{Q}^N$ that is defined at a discrete set of time steps $0 = t_0 < t_1 < \ldots < t_N = T$. Moreover, the Markovian structure of differential equations can be highlighted by expressing the measure as follows
\begin{align}
\mathbb{Q}^N(\de a, \de u_0, \ldots, \de u_N) = \mathcal{Q}_0(\de u_0) \mathcal{Q}_a(\de a) \prod_{n = 0}^{N-1} \mathcal{K}(\de u_{n+1}|u_n,a),
\end{align}
where $u_i(\omega) = u(t_i,\omega) \in \R^d$ and $\mathcal{Q}_a$ and $\mathcal{Q}_0$ are the marginal distributions of $a(\cdot,\cdot,\omega)$ and $u_0(\omega)$, respectively, in which we assume they are independent. Here, $\mathcal{K}(\de u_{n+1}|u_{n},a)$ is the transition probability kernel which models the forward evolution of the ODE. The evolution of ODEs is deterministic, so this kernel is typically a delta function centered at the solution of the next time step. Since ODEs are only solved approximately, $\mathbb{Q}^N$ cannot be truly accessed and can only be approximated by some approximate measure $\tilde{\mathbb{Q}}^N$.

\begin{example} \label{ex:ode}
Consider the random ODE from the third ODE example in \cite{yang2023context}
\begin{align} \label{eq:ex1}
\begin{cases}
    \frac{d}{dt}u(t,\omega) = \gamma_1(\omega) c(t,\omega) u(t,\omega) + \gamma_2(\omega), \\
    u(0,\omega) = u_0(\omega)
    \end{cases}
\end{align}
where $u(t,\omega) \in \R$ is a function of time, $c(t,\omega)$ is a Gaussian process, and $\gamma_1(\omega)$, $\gamma_2(\omega)$, and $u_0(\omega)$ are $\R$--valued uniform random variables. The random function $a(t,u,\omega) = \gamma_1(\omega) c(t,\omega) + \gamma_2(\omega)$ is parametrized by $\gamma_1, \gamma_2$, and $c(t,\omega)$. In practice, Gaussian processes are approximated by finite dimensional Gaussian vectors $\mathbf{c}\in \R^{\mathbf{N_c}}$, and so the random function $a$ is modeled by a finite dimensional random vector. The probability measure that arises from \eqref{eq:ex1} can be represented in terms of a \emph{density} $Q^N$ 
\begin{align}
    \mathbb{Q}^N&(d\gamma_1,\de \gamma_2,\de \mathbf{c}, \de u_0,\de \mathbf{u}) = Q^N(\gamma_1,\gamma_2,\mathbf{c}, u_0, \mathbf{u}) \de \mathbf{W} \\
    & = Q_{\gamma_1}(\gamma_1) Q_{\gamma_2}(\gamma_2) Q_{\bf{c}}(\bfc) \prod_{n = 0}^{N-1}K(u_{n+1}| u_n, \mathbf{c}, \gamma_1,\gamma_2) \de \mathbf{W}, \nonumber
\end{align}
where $\bf{u} \in \R^N$,  $\mathbf{u}_i = u_i$, and  $d\mathbf{W} = d\gamma_1 \,d\gamma_2\,d\mathbf{c}\,du_0\,d\mathbf{u}$. Here, $K$ is a Dirac measure concentrated at the deterministic solution update.

\end{example}

\subsection{Random boundary value problems}
Here we consider random boundary value problems (BVPs). The solution no longer evolves in time, however $V$ and $W$ are function spaces defined over some domain $E \subset \R^d$, for example, $V = L^2(E)$ and $W = L^2(\partial E)$. Generally, we may consider equations of the form 
\begin{align}
\begin{cases}
    &\mathcal{A}[u(x,\omega),\omega] = F(\omega),\, x \in E \\
    &\mathcal{B}[u(x,\omega),\omega] = G(\omega),\, x \in \partial E
    \end{cases}
\end{align}
where for a fixed $\omega$, the solution $u(\cdot,\omega): E \to \R$ is an element of $V$ in the domain while $u(\cdot,\omega): \partial E \to \R$ is an element of $W$ on the boundary. For example, for a steady-state random reaction-advection-diffusion equation, $V = L^2(E)$, $W = L^2(\partial E)$, $\mathcal{A}[u(x,\omega),\omega] = -\nabla \cdot(\kappa(x,\omega) \nabla u(x,\omega) ) + r(u(x,\omega),\omega) + \nabla \cdot f(x,u(x,\omega),\omega)$, and $\mathcal{B}[u(x,\omega),\omega] = u(x,\omega) $ where $\kappa(x,\omega)$ is the permeability function, $r(u,\omega)$ is the reaction function, $f(x,u,\omega)$ is the flux function, $F(\omega) = s(x,\omega)$ is the source term, and $G(\omega) = g(x,\omega)$ is the boundary data. 

\begin{example}\label{ex:bvprd}
    Consider the random reaction-diffusion equation on $E = [0,1] \subset \R$ from Example 13 in \cite{yang2023context}, which has the form 
    \begin{align} \label{eq:exbvp}
        -0.1 a(\omega) \frac{d^2}{dx^2} u(x,\omega) + k(x,\omega) u(x,\omega) = c(x,\omega), \, x\in [0,1] \\
        u(0,\omega) = u_0(\omega), u(1,\omega) = u_L(\omega). 
    \end{align}
Here, coefficient $a(\omega)$ and boundary terms $u_0(\omega), u_L(\omega)$ are $\R$--valued random variables, $u(x,\omega)\in \R$ for $x \in [0,1]$. The functions $k(x,\omega), c(x,\omega)$ are modeled with Gaussian processes, which can be approximated by finite dimensional Gaussian vectors $\mathbf{k} \in \R^{N_k}$, $\mathbf{c} \in \R^{N_c}$.  Consider a discretization of the domain over points $0 = x_0< x_1< \cdots< x_{N} = 1$. The resulting probability measure that arises from \eqref{eq:exbvp} can be represented in terms of a density
\begin{align}
    \mathbb{Q}^N(\de a, \de \mathbf{c} , \de \mathbf{k}, \de u_0, \de u_L, \de \mathbf{u}) = Q^N(  a,  \mathbf{c} , \mathbf{k},  u_0,  u_L, \mathbf{u}) \de \mathbf{W},
\end{align}
where $\bfu(\omega) \in \R^N$, $\bfu_i(\omega) = u(x_i, \omega)$, and $\de\mathbf{W} = \de a\de\mathbf{c}\de\mathbf{k}\de u_0 \de u_L \de\mathbf{u}$. Solutions to boundary value problems are typically approximated by finite difference or element schemes, where the solution at a particular stencil point or element is explicitly dependent only on a local neighborhood of the point or element. In this example, one would typically use a central difference scheme so that the density can be expressed with the following conditional independence structure
  \begin{align}
        Q^N(a,\bfc,\mathbf{k},u_0, u_L,\bfu) = Q_a(a) Q_\bfc(\bfc) Q_{\mathbf{k}}(\mathbf{k}) Q_0(u_0) Q_L(u_L)  \times \prod_{n = 1}^{N-1} K(u_n|u_{n-1},u_{n+1},a,\bfc,\mathbf{k} ).
    \end{align}
The kernel function $K$ here models a \emph{Markov random field}, so the dependence structure is not sequential but is a graphical model. This kernel function is still a Dirac measure since the dependence of the solutions on the stencils given fixed parameters is deterministic. While we do not explicitly explore this structure in this paper, it will be of interest for future work. There is extensive work on how graph structure enhances inference, uncertainty quantification, and generative modeling \cite{Birmpa2022PGM}, which can be explored to enhance operator learning. Autoregressive models also exhibit graph structure.

\end{example}

\subsection{Random PDEs}
We consider the class of conservation law PDEs which are time-evolving but have no boundary conditions. In \eqref{eq:generalpde}, we have $V = BV(\R^d)$, where $BV$ is the space of functions of bounded variation on $\R^d$.  A random conservation law has the form 
\begin{align}
\begin{cases}
    \frac{\partial}{\partial t} u(t,x,\omega) + \nabla \cdot f(u(t,x,\omega), \omega) = 0 \\
    u(0,x,\omega) = u_0(x,\omega) \nonumber
    \end{cases}
\end{align}
where $f(u,\omega)$ is the random flux function taking value in $V$. Here $x \in \R^d$ and $u(x,\omega) \in \R$. The initial condition $u_0(x,\omega)$ is a random function on $\R^d$ taking value in $\R$. This example is similar to the random ODE in Section~\ref{sec:randomode}, except that $u(t,\cdot,\omega)$ is now a $BV(\R^d)$ space valued stochastic process rather than one on $\R^d$. 

\begin{example}
The parametrized family of conservation laws considered in $\cite{yang2024pde}$ models the flux function as a polynomial in $u$, $f(u,\omega) = a(\omega) u^3 + b(\omega) u^2 + c(\omega) u,$. The coefficients $a(\omega), b(\omega), c(\omega)$ are $\R$--valued uniform random variables and the initial condition is a Gaussian process. The domain is $[0,1]\subset \R$. Suppose we discretize time $0 = t_0<t_1<\cdots< t_N = T$ and space  $0= x_1 <\ldots< x_M = 1$ and define the collection of random vectors $\{\mathbf{u}_n(\omega) \}_{n = 0}^N$ such that $\mathbf{u}_n^m(\omega) = u(t_n,x_m,\omega)$. The probability measure associated with this random conservation law has both a Markov property and a graph structure in space when discretized. We have that the measure and density are 
\begin{align}
\mathbb{Q}^{N,M}(\de a,\de b,\de c,\de\bfu_0,\ldots,\de\bfu_N) &=  Q^{N,M}( a,b,c, \bfu_0, \ldots, \bfu_N) \de \mathbf{W}  \nonumber \\
   & = Q_{a,b,c}(a,b,c) Q_0(\bfu_0) \prod_{n = 0}^{N-1} K(\bfu_{n+1}|\bfu_n, a,b,c) d\mathbf{W}\nonumber \\
    & = Q_{a,b,c}(a,b,c) Q_0(\bfu_0) \prod_{n = 0}^{N-1} \prod_{m = 1}^{M-1} K(\bfu_{n+1}^m|\bfu_n^{m-1},\bfu_n^m, \bfu_n^{m+1}, a,b,c)d\mathbf{W},
\end{align}
where $\de\mathbf{W} = \de a \de b \de c \de \mathbf{u_0} \ldots, \de\mathbf{u}_N$, $\mathbf{u}_i \in \R^{M+1}$.
\end{example}

The random differential equations formalism describes the data generating process for foundation models of differential equations that are established through operator learning. The crucial aspect is that the data generating measure is one that is defined on Banach space valued random variables. In the next section, we study in-context operator learning in this probabilistic setting.

\section{Probabilistic formulation of in-context operator learning}\label{sec:probicon}

We first provide an abstract probabilistic formulation that is suitable for understanding ICON, describe its training procedure, and describe its connections to Bayesian inference. In-context learning in the setting of language has been studied probabilistically \cite{xie2021explanation}, and transformers have been noted to perform Bayesian inference \cite{muller2022transformers}. Since ICON applies in-context learning to inputs and outputs of differential equation, a probabilistic understanding of it requires describing ICON in terms of Banach and Hilbert space valued random variables. Placing ICON in a functional analytic framework also connects it to other operator learning theories, for example by enabling the study of whether it exhibits discretization invariance \cite{kovachki2023neural}.

\subsection{In-context operator learning and its training procedure}

We model the parameters, conditions, and quantities-of-interest in differential equation as \emph{random variables.} These random variables take values in Banach and Hilbert spaces as inputs and solutions to random differential equations are elements of function spaces. For full generality and rigor, let $A, Y$ be separable Banach spaces and $Z$ to be a separable Hilbert space\footnote{Separability allows for defining regular conditional probabilities and the disintegration of measures \cite{bogachev2007measure,vakhania2012probability,kallenberg2021foundations}.}. We refer the reader to \cite{vakhania2012probability,ledoux2013probability,bogachev2007measure,kallenberg2021foundations} for background on probability theory on Banach spaces. The Hilbert space structure of $Z$ is crucial for establishing the link between least squares predictors and conditional expectations, which we discuss in Theorem~\ref{thm:ICON_conditional}. Since these random variables are Banach space valued, integration with respect to their probability measures are understood as Bochner integrals \cite{bogachev2007measure}. 

Denote the parameter, condition, and QoI random variables to be $\alpha: \Omega \to A$, $y: \Omega \to Y$, and $z: \Omega \to Z$, respectively. The random variables have a joint probability measure $\mathbb{P}_{\alpha, y, z}$ that factorizes according to the following conditional independence structure 
\begin{align}
    \mathbb{P}_{\alpha, y, z} = \mathbb{P}_\alpha \otimes \mathbb{P}_{y} \otimes \mathbb{P}_{z| y, \alpha}.
\end{align}
This emphasizes the structure of RDEs where the conditions and parameters are drawn independently, while the QoI is generated by solving the RDE given fixed conditions and parameters. 

Depending on the specific problem at hand, different components of an RDE corresponds with different random variables in ICON. For instance, for the \emph{forward} problem in Example~\ref{ex:ode}, the parameters $\alpha(\omega) = (\gamma_1(\omega),\gamma_2(\omega))$, the conditions are $y(\omega)= (c(t,\omega), u_0(\omega))$, and the QoIs are the solutions $z(\omega) = u(t,\omega)$. While for inverse problems, the conditions and QoIs are exchanged --- $y(\omega) = (u_0(\omega),u(t,\omega))$, $z(\omega) = c(t,\omega)$.

During the data generation step, we produce i.i.d. samples from the joint distribution 
\begin{align}\label{eq:trainingdata}
    \left\{\left\{ (\alpha_m, y_m^j, z_m^j) \right\}_{j = 1}^J \right\}_{m = 1}^M \sim \mathbb{P}_{\alpha,y,z} = \mathbb{P}_\alpha \otimes \mathbb{P}_{y} \otimes \mathbb{P}_{z| y, \alpha}.
\end{align}
Sampling from the joint measure is done in a two step process: first $M$ operators, represented by the parameters $\alpha_{m}$, are sampled independently from $\mathbb{P}_\alpha$. For each parameter $\alpha_m$, $J$ conditions $y^{j}_m$ are sampled from $\mathbb{P}_y$. For each condition, the QoI $z^j_m$ is sampled from the conditional measure $Q(z|y^j_m,\alpha_{m})$.  For certain differential equations, QoIs are deterministically coupled to $y^j_m$ and $\alpha_{m}$, so the conditional measure is sampled by solving the ODE or PDE at hand. During the training of ICON, however, information about the parameters is omitted, so ICON is trained on data from the marginal distribution $\mathbb{P}_{y,z}$:
\begin{align}\label{eq:yz_marginal_samples}
    \left\{ \left\{(y^{j}_m, z^{j}_m) \right\}_{j = 1}^J \right\}_{m = 1}^M  \sim \mathbb{P}_{y,z}.
\end{align}
For differential equations, ICON is a model that learns the solution operator of a differential equation given only examples condition-QoI pairs without additional training. Formally, ICON is a mapping
\begin{align} \label{eq:icon_mapping}
    \mathcal{T}_\theta: Y \times (Y\times Z)^{J-1} \to Z 
\end{align}
whose inputs is a set of example condition-QoI pairs $\{(y^j_m, z^j_m)\}_{j = 1}^{J-1}  \sim (Y\times Z)^{J-1}$ corresponding to the same parameter $\alpha_m$, along with a \emph{new} condition $y_m^J\in Y$ and outputs a prediction for the QoI $z^{J}_m \in Z$ that best corresponds with the condition. The example condition-QoI pairs informs the `context' which is described by the parameter $\alpha^m$. The ICON mapping $\mathcal{T}_\theta$ has a transformer architecture, which has been the basis for the success of large language models today \cite{radford2019language}. The connections between transformers, in-context learning, and Bayesian inference has been noted in previous work \cite{muller2022transformers, ye2024exchangeable}.

ICON has been implemented with transformers with encoder-decoder architectures \cite{yang2023context}, where the encoder portion understands the context based on the example pairs, and the decoder portion makes the prediction. Decoder only architectures, however, have also been found to be empirically successfully \cite{yang2023fine,yang2024pde}. A key strength of transformer networks is that they are able to take arbitrary number of inputs as arguments. 

The ICON transformer is trained by solving the least squares problem
\begin{align}\label{eq:icontraining}
    \min_\theta \frac{1}{M} \sum_{m = 1}^M \left\| z_m^J - \mathcal{T}_\theta\left( y_m^J ; \{y_m^j, z_m^j\}_{j = 1}^{J-1} \right)\right\|_Z^2.
\end{align}

In the next section, we show that the minimizer of this least squares problem approximates a conditional expectation, in particular, the mean of a posterior predictive distribution which arises in Bayesian inference.

\begin{remark}{(Probabilistic operator learning is agnostic to forward and inverse problems.) } \label{rmk:probopinf}
As formulated in \cite{kovachki2024operator}, deterministic operator learning aims to learn a mapping from parameters $\alpha$ and conditions $y$ to QoI solutions $z$ in which, for well-posed problems, the mapping is well-defined and is unique. This approach, however, has limitations for inverse problems, where multiple QoIs may correspond to the same parameter and condition. In the probabilistic interpretation, forward and inverse problems are addressable in the same framework. In Section~\ref{sec:icon_bayes} we will see that the Bayesian interpretation explains why ICON has been successful for both forward and inverse problems. 
  
\end{remark}

\subsection{Probability measures of in-context operator networks}

The key observation that reveals how ICON works is to see that the condition-QoI pairs that are evaluated by ICON $\{(y_m^j,z_m^j)\}_{j = 1}^{J}$ for a fixed $m$ are \emph{not} independent samples of $\mathbb{P}_{y,z}$. Rather, as they arise from the same parameter $\alpha_m$, they are independent samples of the conditional distribution $\mathbb{P}_{y,z}(\cdot|\alpha) = \mathbb{P}_{z|y,\alpha} (\cdot | y,\alpha) \otimes \mathbb{P}_{y}$. This dependence explains why example condition-QoI pairs, $\{(y^j, z^j )\}_{j = 1}^{J-1}$, informs the prediction of QoIs $z^J$ given new conditions $y^J$. 

To see this more clearly, the following probability distribution is another expression for the measure that generates the training data 
\eqref{eq:trainingdata}
\begin{align} \label{eq:jointparamex}
    \mathbb{Q}_{\alpha, y^1, z^1, \ldots y^J, z^J } = \mathbb{P}_{\alpha}\otimes \left(\mathbb{P}_{y^1} \otimes \mathbb{P}_{z^1 | y^1, \alpha}\right) \otimes \cdots \otimes \left(\mathbb{P}_{y^J} \otimes \mathbb{P}_{z^J | y^J, \alpha}\right).
\end{align}
Each tuple $\left(\alpha_m,\left\{ (y_m^j, z_m^j)\right\}_{j = 1}^J\right)$ is a single sample from $\mathbb{Q}_{\alpha, y^1, z^1, \ldots y^J, z^J }$. 
Figure~\ref{fig:icongraph} shows the graphical model that describes the dependencies between the parameter and $J$ condition-QoI pairs. ICON learns how to detect context by omitting $\alpha$ during training. Therefore, ICON only trains on data from the marginal distribution on $(Y\times Z)^{J}$
\begin{align}\label{eq:iconjointmeasure}
   \mathbb{Q}^J(\cdot) \coloneqq \mathbb{Q}_{y^1, z^1, \ldots, y^J, z^J}(\cdot) = \int \mathbb{Q}_{\alpha, y^1, z^1, \ldots, y^J, z^J}(d\alpha, \cdot).
\end{align}
We refer this measure as the \emph{ICON joint distribution} as it generates the training data for ICON. Crucially, this is distinct from the measure \eqref{eq:yz_marginal_samples} as the collection of $J$ example condition-QoI pairs all correspond to the same $\alpha$ while samples from \eqref{eq:yz_marginal_samples} are generated independently. From the joint distribution, we can derive the conditional distribution 
 \begin{align}\label{eq:probabilistic_ICON}
 \mathbb{Q}_{z^{J} \big| y^{J}, \{y^{j}, z^{j} \}_{j = 1}^{J-1}}\left(\,\cdot\, \big| y^{(J)}, \{y^{(j)}, z^{(j)} \}_{j = 1}^{J-1} \right),\end{align} 
 which is a probability distribution on $Z$. We refer this as the \emph{ICON conditional distribution}, which is the distribution of the predictions $z^J$ given condition $y^J$ and context informed by the example condition-QoI pairs $\{ (y^j, z^j) \}_{j = 1}^{J-1}$. The next theorem shows that ICON approximates the expectation of the conditional distribution, which is a point estimate for this distribution.

\begin{theorem}
    \label{thm:ICON_conditional} 
    Assume $z^J \in L^2(\mathbb{Q}^J; Z)$, i.e.\ $\Ex_{\mathbb{Q}^J}[\|z^J\|_Z^2] < \infty$, and let 
$\Sigma := \sigma\!\big(y^J,\{(y^j,z^j)\}_{j=1}^{J-1}\big)$. The mapping $\mathcal{T}^\star: Y\times (Y\times Z)^{J-1} \to Z$, where $\mathcal{T}^\star$ is the conditional expectation with respect to \eqref{eq:probabilistic_ICON} 
    \begin{align}  \label{eq:conditional_expectation_ICON} \mathcal{T}^\star\left(y^J; \{(y^j,z^j) \}_{j = 1}^{J-1} \right) \coloneqq \Ex\left[z^J \big| y^J, \{(y^j,z^j) \}_{j = 1}^{J-1} \right]\end{align}
is the unique minimizer of the mean-squared error with respect to \eqref{eq:iconjointmeasure}
\begin{align}
\mathcal{T}^\star 
= \arg\min_{\mathcal{T} \in L^2(\Sigma; Z)} 
\Ex_{\mathbb{Q}^J}\!\left[
\left\| z^J - \mathcal{T}\!\left(y^J;\{(y^j,z^j)\}_{j=1}^{J-1}\right)\right\|_Z^2
\right],
\end{align}
where $L^2(\Sigma;Z):=\{\,\mathcal{T}(y^J,\{(y^j,z^j)\}_{j=1}^{J-1}): \mathcal{T} \text{ is }\Sigma\text{-measurable and } \Ex\|\mathcal{T}\|_Z^2<\infty\,\}$.
\end{theorem}
\begin{proof}
This result follows from the Hilbert space projection theorem and the characterization of conditional expectation as an $L^2$-projection; see Theorem 1.35 and Section 8 in \cite{kallenberg2021foundations}. \end{proof}

When the least squares problem is approximated through a Monte Carlo approximation, we obtain the training objective in \eqref{eq:icontraining}. This perspective that ICON approximates the conditional expectation explains why it can be stably trained for both forward and inverse problems. As we discussed in Remark~\ref{rmk:probopinf}, deterministic operator learning does not naturally address problems where there are multiple quantities of interest that correspond to the same conditions and parameters. This probabilistic interpretation of ICON explains that it is able to approximate an expected value over the distribution of QoI outputs.

\begin{remark}
    ICON is capable of training with variable number of example-QoI pairs \cite{yang2023context,yang2023fine,yang2024pde}, i.e., the context size $J$ does not need to be fixed. This flexibility is due to the transformer architecture, which can process an arbitrary number of example pairs. While we describe our probabilistic framework under the assumption of a fixed $J$, it can be extended to accommodate for a variable $J$ as well. We leave a detailed treatment of this extension to future work. 
\end{remark}

\begin{figure*}
    \centering
    \includegraphics[width=\linewidth]{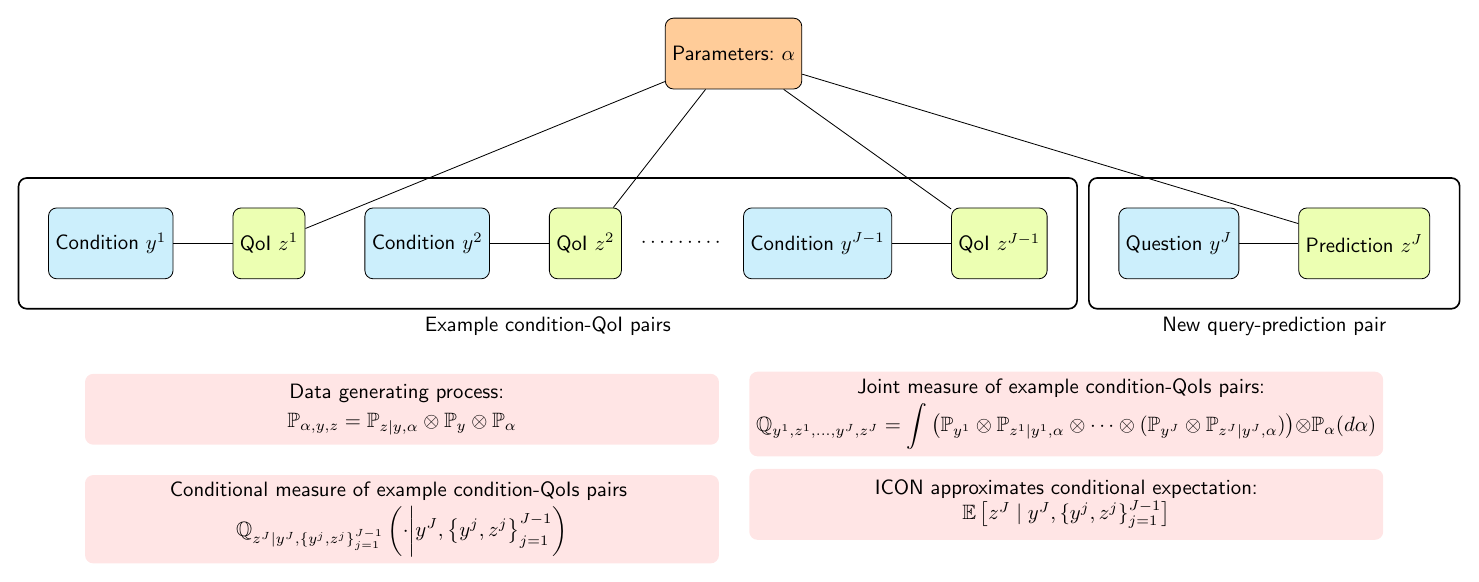}
    \caption{The ICON graphical model, corresponding to the probability measure described in \eqref{eq:jointparamex}. In practice, ICON is trained on condition-QoI pairs in the left box to make predictions given a new condition. The example condition-QoI pairs informs `context,' which is instantiated as hidden operator parameters that is learned through the example pairs. When producing the ICON training data, the parameters and conditions are generated independently given a fixed parameter. During training, the parameters are omitted, and the data organized as samples from the joint measure of example condition-QoI pairs. ICON approximates the conditional expectation of the conditional measure that arises from the joint measure. }
    \label{fig:icongraph}
\end{figure*}

\begin{remark}[Probabilistic formulation of classical operator learning]
To clarify the distinction between "classical" operator learning methods --- such as DeepONets \cite{lu2019deeponet}, Fourier neural operators \cite{li2021fourier}, and neural operators more broadly \cite{kovachki2023neural}--- and in-context operator learning, we briefly discuss the former in our probabilistic context. Classical operator learning also uses training data of the form \eqref{eq:trainingdata} with $J = 1$, but the key difference is that the operator parameters $\alpha_m$ are observed during training and serve as inputs to the operator. Formally, classical operator learning seeks a mapping
\begin{align}
    \mathcal{S}_\theta: Y \times A \to Z
\end{align}
trained on the least squares problem 
\begin{align}
    \min_\theta \frac{1}{M}\sum_{m = 1}^M \left\| z_m - \mathcal{S}_\theta(y_m, \alpha_m)\right\|_Z^2.
\end{align}
In the probabilistic terms, the optimal solution is the conditional expectation of the distribution $\mathbb{Q}_{z | y,\alpha}$, i.e., $\mathcal{S}^\star(y,\alpha) = \Ex[z | y,\alpha]$. 
\end{remark}

\subsection{In-context operator learning implicitly performs Bayesian inference}
\label{sec:icon_bayes}

We now explain how the ICON conditional distribution \eqref{eq:probabilistic_ICON} and the conditional expectation that ICON approximates in Theorem~\ref{thm:ICON_conditional} arises in Bayesian inference. Previous work \cite{falck2024context,muller2022transformers,ye2024exchangeable} have explained in-context learning and transformers as being able to approximate the expectation of the \emph{posterior predictive distribution} (PPD). Our work formulates in-context learning infinite dimensions, which is necessary for solutions of differential equations. 
 
 To see ICON's connection to Bayesian inference, consider the classical Bayesian approach for solving the abstract problem ICON addresses: given only example pairs $\left\{(y^j,z^j)\right\}_{j = 1}^{J-1}$, and a new condition $y^J$, we wish to predict $z^J$. Given a prior on the parameters $\mathbb{P}_\alpha$ and a likelihood model $\mathbb{P}_{z | y, \alpha}$, Bayesian inference produces a posterior distribution on the parameters, which models the distribution of parameters that is likely to have generated the example data, 
 \begin{align}
    \underbrace{\mathbb{P}_{\alpha |  \left\{y^{j}, z^{j}\right\}_{j=1}^{J-1} }\left(\cdot \bigg| \left\{y^{j}, z^{j}\right\}_{j=1}^{J-1}\right) }_{\text{Posterior}} 
    \propto \mathbb{P}_\alpha(\cdot) \mathbb{P}_{ \left\{y^{j}, z^{j}\right\}_{j=1}^{J-1} | \alpha} = \underbrace{\mathbb{P}_\alpha(\cdot)}_{\text{Prior}} \,\underbrace{\prod_{j=1}^{J-1} \mathbb{P}_{y^j,z^j|\alpha}\left(y^{j}, z^{j} \bigg| \alpha\right)}_{\text{Likelihood function}}
 \end{align} 

In the language of in-context learning, the Bayesian posterior describes the `context' informed by the example condition-QoI pairs. The distribution of \emph{new} predictions under the posterior distribution, known as the \emph{posterior predictive distribution}, is obtained by integrating the likelihood with respect to the posterior distribution. The posterior predictive is the distribution of QoI $z^J$ conditioned on the example pairs $\{y^j, z^j \}_{j = 1}^{J-1}$ and the new condition $y^{(J)}$ 
\begin{align}\label{eq:iconbayes}
    \mathbb{Q}_{z^J | y^J, \{ (y^j, z^j)\}_{j = 1}^{J-1} }\left(\cdot \big| y^J, \{(y^j,z^j) \}_{j = 1}^{J-1}\right) = \int \underbrace{\mathbb{P}_{z^J | y^J, \alpha}(\cdot | y^J,\alpha)}_{\text{Likelihood}} \underbrace{\mathbb{P}_{\alpha | \{ (y^j, z^j)\}_{j = 1}^{J-1} }(d\alpha|\{ (y^j, z^j)\}_{j = 1}^{J-1} ) }_{\text{Posterior}}. 
\end{align}
Notice that the posterior predictive distribution \eqref{eq:iconbayes} is identical to the ICON conditional measure \eqref{eq:probabilistic_ICON}. Moreover, from Theorem~\ref{thm:ICON_conditional}, we know that ICON approximates the mean of \eqref{eq:probabilistic_ICON}, so this implies that ICON \emph{approximates the mean of the posterior predictive distribution}.

The Bayesian perspective, however, requires a prior distribution $\mathbb{P}_\alpha$ and a likelihood model $\mathbb{P}_{y,z|\alpha}$ to be defined \emph{a priori}, meaning that the posterior predictive distribution may change as one changes the prior or likelihood. This fact appears to be in conflict with what ICON accomplishes --- that is, ICON makes no assumption on the prior and does not require explicit access to a likelihood model. These objects are, in fact, handled \emph{implicitly} in ICON. The choice of prior and likelihood is actually made when generating the training data \eqref{eq:trainingdata}, i.e., the choice is being made when curating the collection of differential equations and conditions. This means that ICON has access to the prior and likelihood only through the samples. \cite{ye2024exchangeable} argues that in-context learning and transformers perform a type of empirical Bayes method, where pre-training is akin to learning the prior distribution, and fine-tuning performs posterior inference and prediction.

Moreover, as ICON directly approximates the posterior predictive distribution, it often does not explicitly work with the true latent variable $\alpha$. Rather, ICON implicitly learns a latent variable $\zeta = g(\alpha)$ that may be a function of the true parameter $\alpha$. This is most clear for ICON based on encoder-decoder transformer architectures \cite{yang2023context} where the input is encoded into some latent space. There is no unique latent variable --- many latent variables can produce the same posterior predictive. That is, we may more generally have 
\begin{align}\label{eq:postpred_differentlatents}
    \mathbb{Q}_{z^J| y^J, \{ (y^j, z^j)\}_{j = 1}^{J-1} }\left(\cdot \big| y^J ;  \{y^j, z^j \}_{j = 1}^{J-1} \right)  = \int {\mathbb{P}_{z^J| y^J, \zeta}(\cdot | y^J,\zeta)} {\mathbb{P}_{\zeta | \{y^j, z^j\}_{j = 1}^{J-1} }(d\zeta|  \{y^j, z^j\}_{j = 1}^{J-1})}.
\end{align}
Regardless of the choice of latent variables, this perspective shows ICON performs Bayesian inference \textit{implicitly} by a direct and explicit characterization of the conditional expectation of the posterior predictive distribution.

\paragraph{ICON implicitly performs amortized and likelihood-free inference}

We emphasize that ICON actually performs Bayesian inference in a \emph{amortized} and \emph{likelihood-free} manner. Inference involves learning the posterior distribution of parameters given observations. This process must be repeated whenever new observations arrive, which can be computationally expensive. \emph{Amortized inference} addresses this by learning a generative model that maps observations directly to posterior distributions \cite{ganguly2023amortized}. Although training such a model may be more costly than a single inference step, the expense is amortized across many possible future queries. Likelihood-free or \emph{simulation-based} inference methods \cite{Cranmer2019SBI, baptista2024conditional,hosseini2025conditional} characterize the posterior distribution over model parameters without requiring evaluation or even knowledge of the likelihood. One approach to amortized and likelihood-free inference is to construct a block-triangular generative model for the joint parameter-observation distribution and then exploit the structure of the generative model to sample from the conditional distribution of the parameters given observations \cite{baptista2024conditional}.

ICON is an example of amortized and likelihood-free inference in a distinct way: rather than repeatedly recomputing or even explicitly characterizing the posterior distribution, it maps observations directly to the posterior predictive distribution, bypassing the bottleneck of computing the posterior entirely. This is an advantage when one simply wants a predictive model without the need to fully characterize the posterior. Posterior distributions are often difficult to characterize, particularly when they are supported on a low-dimensional manifold or are highly nonconvex, which makes sampling from them difficult. In Section~\ref{sec:numerical}, we consider an ODE that exhibits \emph{non-identifiability}, meaning that multiple parameters may produce the same observational data. The resulting posterior then lies on a lower-dimensional manifold that is non-trivial to characterize. ICON bypasses all of these challenges by approximating the mean of the posterior predictive directly. We summarize the Bayesian explanation of deterministic in-context operator learning \cite{yang2023context,yang2023fine} in Figure~\ref{fig:iconbayes}. In Section~\ref{sec:genicon}, a generative formulation of ICON will be able to characterize uncertainties of the predictive model.

\begin{figure*}
    \centering
    \includegraphics[width = \textwidth]{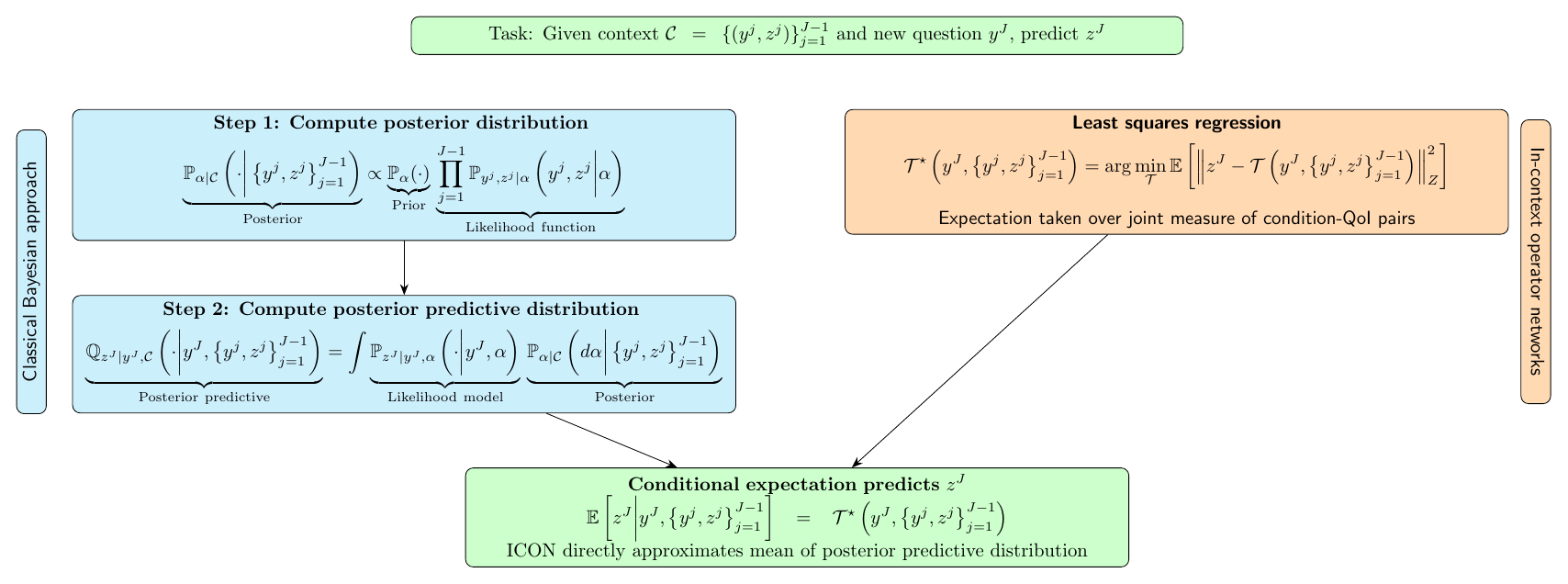}
    \caption{ICON is a model that performs Bayesian inference. Given a collection of example condition-QoI pairs and a new condition $y^J$, we predict $z^J$. The classical Bayesian approach first infers the model parameters from the example pairs and then characterizes the posterior predictive distribution of outputs and its mean. ICON approximates the mean in a single step by bypassing computing the posterior distribution. See Section~\ref{sec:nonidentifiable} for an explicit illustration of how ICON bypasses the need to sample from a degenerate posterior distribution. }
    \label{fig:iconbayes}
\end{figure*}

\section{Generative in-context operator learning} \label{sec:genicon}
We have shown that ICON approximates the expectation of the posterior predictive distribution of a QoI $z^J$ given a new condition $y^J$ and example condition-QoI pairs. A natural generalization is a \emph{generative model} that produces samples from the posterior predictive distribution. Generating from the PPD provides a way of \emph{quantifying uncertainties} in model predictions.

\subsection{Posterior predictive sampling via generative ICON} 

Producing samples from the posterior predictive is an instance of \emph{conditional generative modeling}. Let $X,Y,Z$ be separable Banach spaces, and let $x \sim \mathbb{P}_x$ be a reference random variable on $X$ so that $\mathbb{P}_x$ is easy to sample, e.g., a Gaussian random variable. Let $(y,z) \sim \mathbb{P}_{y,z}$ be a random variable with joint distribution $\mathbb{P}_{y,z}$ and let $\mathbb{P}_{z|y}(\cdot |y)$ denote the conditional distribution of $z$ given $y$. The task of conditional generative modeling is to construct a map $\mathcal{G}: X\times Y \to Z$ such that for each $y \in Y$, the pushforward of $\mathbb{P}_x$ under the map $\mathcal{G}(\cdot, y)$ equals the conditional distribution of $z$ given $y$: $\mathcal{G}(\cdot,y)_\sharp \mathbb{P}_x = \mathbb{P}_{z|y}(\cdot|y).$  The map is learned through data $\{(y_n,z_n)\}_{n = 1}^N$, where there are many algorithms that accomplish this task \cite{mirza2014conditional,baptista2024conditional,hosseini2025conditional,jacobsen2025cocogen}. The resulting map $\mathcal{G}$ is called a conditional generative model. 

The following theorem states the existence of a conditional generative model, which we refer to as \emph{generative ICON} (GenICON), that samples from the Bayesian posterior predictive distribution \eqref{eq:iconbayes}.

\begin{theorem}[Existence of Generative ICON]\label{thm:genicon_existence}

    Let $\eta \sim \mathbb{P}_\eta$ be a non-atomic reference random variable on separable Banach space H. There exists a measurable map
    \begin{align}
        \mathcal{G}: H \times Y \times (Y\times Z)^{J-1} \to Z
    \end{align}
    such that for every $(y^J, \{y^j,z^j \}_{j = 1}^{J-1})$, the pushforward measure of $\mathbb{P}_x$ under $\mathcal{G}(\cdot,y^J, \{ y^j,z^j\}_{j = 1}^{J-1})$ is equal to the posterior predictive distribution \eqref{eq:iconbayes}
    \begin{align*}
        \mathcal{G}(\cdot, y^J, \{y^j,z^j \}_{j = 1}^{J-1})_\sharp \mathbb{P}_x = \mathbb{Q}_{z^J|y^J,\{y^j,z^j \}_{j = 1}^{J-1}}. 
    \end{align*}
\end{theorem}
This statement follows from standard results in probability theory; see, for example, \cite{kallenberg2021foundations,bogachev2007measure}.
\begin{proof}
    The space $H$ being a separable Banach space implies it is Polish. By the isomorphism theorem for non-atomic probability measure (Theorem 9.4.7 in \cite{bogachev2007measure}, Theorem 2.18 in \cite{kallenberg2021foundations}), any non-atomic probability measure on $H$ can be mapped by a measurable bijection to the uniform measure on $[0,1]$ \footnote{See also the Borel isomorphism Theorem 1.8 of \cite{kallenberg2021foundations}. }. Therefore, there exists a measurable map $\tau: H \to [0,1]$ such that $\tau_\sharp \mathbb{P}_\eta = \mathbb{U}_{[0,1]}$, the uniform distribution on $[0,1]$. Then, since the posterior predictive distribution is a probability kernel, the kernel representation lemma (Lemma 4.22 and 3.2(vii) in \cite{kallenberg2021foundations}) provides a measurable map $f: [0,1] \times Y \times (Y\times Z)^{J-1}\to Z$ such that $f(\cdot,y^J, \{ y^j,z^j\}_{j = 1}^{J-1})_\sharp \mathbb{U}_{[0,1]} = \mathbb{Q}_{z^J|y^J,\{y^j,z^j \}_{j = 1}^{J-1}}$. Therefore we can define the generative ICON operator $\mathcal{G}$ by composing $\tau$ and $f$,
\begin{align*}
    \mathcal{G}(\eta, y^J, \{ y^j,z^j\}_{j = 1}^{J-1}) \coloneqq f(\tau(\eta) , y^J, \{ y^j,z^j\}_{j = 1}^{J-1})
\end{align*}
so that by construction, $\mathcal{G}(\cdot, y^J, \{ y^j,z^j\}_{j = 1}^{J-1} )_\sharp \mathbb{P}_\eta = \mathbb{Q}_{z^J|y^J,\{y^j,z^j \}_{j = 1}^{J-1}}$.
\end{proof}

\begin{remark}[Existence of (conditional) generative models in infinite dimensions]  Theorem~\ref{thm:genicon_existence} focuses on conditional generative modeling for the posterior predictive distribution, but the result applies more broadly: it guarantees the existence of (conditional) generative models for probability measures on infinite dimensional spaces. Although the result follows directly from classical results in probability theory, to our knowledge no prior work in infinite-dimensional generative modeling has stated it explicitly. Many recent approaches implicitly rely on Theorem~\ref{thm:genicon_existence} \cite{rahmangenerative,pidstrigach2024infinite,baptista2024conditional,lim2023score}. 
\end{remark}

The GenICON map produces samples from the ICON conditional distribution \eqref{eq:probabilistic_ICON} by mapping samples from a reference distribution $\eta \sim \mathcal{P}_\eta$ and input conditions and context given by $y^J, \{(y^j,z^j)\}_{j = 1}^{J-1}$ to a prediction $z^J$. We refer the original ICON as presented in \cite{yang2023context,yang2023fine,yang2024pde} simply as ICON. Since \eqref{eq:probabilistic_ICON} is precisely the Bayesian posterior predictive of $z^J$, generative ICON is a generative model for it. From the perspective of operator learning, generative ICON does not produce a single operator, rather it produces a family of operators indexed by the reference random variable $\eta$. The connection between the family of generative ICON operators and the original ICON is given in the following theorem.

\begin{theorem}[Mean GenICON is ICON]\label{thm:probicon}
Let $\mathcal{G}$ be a GenICON as defined in Theorem~\ref{thm:genicon_existence}. For each $(y^J, \{y^j,z^j \}_{j = 1}^{J-1}) \in Y\times (Y\times Z)^{J-1},$ the expected value of $\mathcal{G}(\cdot, y^J, \{y^j,z^j \}_{j = 1}^{J-1})$ with respect to $\mathbb{P}_x$ coincides with the conditional expectation in Theorem~\ref{thm:ICON_conditional}
\begin{align}
 \int_X \mathcal{G}(x,y^J, \{y^j,z^j \}_{j = 1}^{J-1}) \mathbb{P}_x(dx) =  \mathcal{T}^\star\left(y^J; \{(y^j,z^j) \}_{j = 1}^{J-1} \right) \coloneqq \Ex\left[z^J \big| y^J, \{(y^j,z^j) \}_{j = 1}^{J-1} \right].
\end{align}
\end{theorem}

\begin{proof}
    This follows immediately by the definition of a pushforward measure. 
\end{proof}

We emphasize that Theorems~\ref{thm:genicon_existence} and~\ref{thm:probicon} are agnostic to the specific choice of generative modeling algorithm. While we train  GenICONs using GANs in our numerical examples in Section~\ref{sec:numerical}, one may consider other generative models on function spaces \cite{lim2023score,hagemann2025multilevel,pidstrigach2024infinite,rahmangenerative}. In Section~\ref{sec:icon_gan} we provide algorithmic details describing the training of a conditional GANs from data generated for ICON.

\paragraph{Generative ICON performs uncertainty quantification.} 
Generative ICON provides uncertainty quantification without additional cost by characterizing the distribution of predictions. By repeatedly sampling from the reference distribution with the same condition $y^J$ and context, GenICON produces an empirical distribution for the posterior predictive distribution from which we can estimate the predictive mean, higher moments, as well as distributions of other observables. Like ICON, GenICON is a likelihood-free and amortized inference method that can produce posterior predictive samples given context without the need to recompute the posterior distribution each time there are new observations. Figure~\ref{fig:icongen} illustrates the capabilities of GenICON in comparison to standard ICON. ICON only produces the posterior predictive mean, which is a single function, while GenICON produces samples which can be used to not only characterize the mean but also the \emph{uncertainty} surrounding the predictive mean.  In a sense, GenICON can be viewed as a generalization of Gaussian processes, which are used to make new predictions given only a new input and data given in ordered pairs \cite{williams2006gaussian}.

\begin{figure*}
    \centering
    \includegraphics[width = \textwidth]{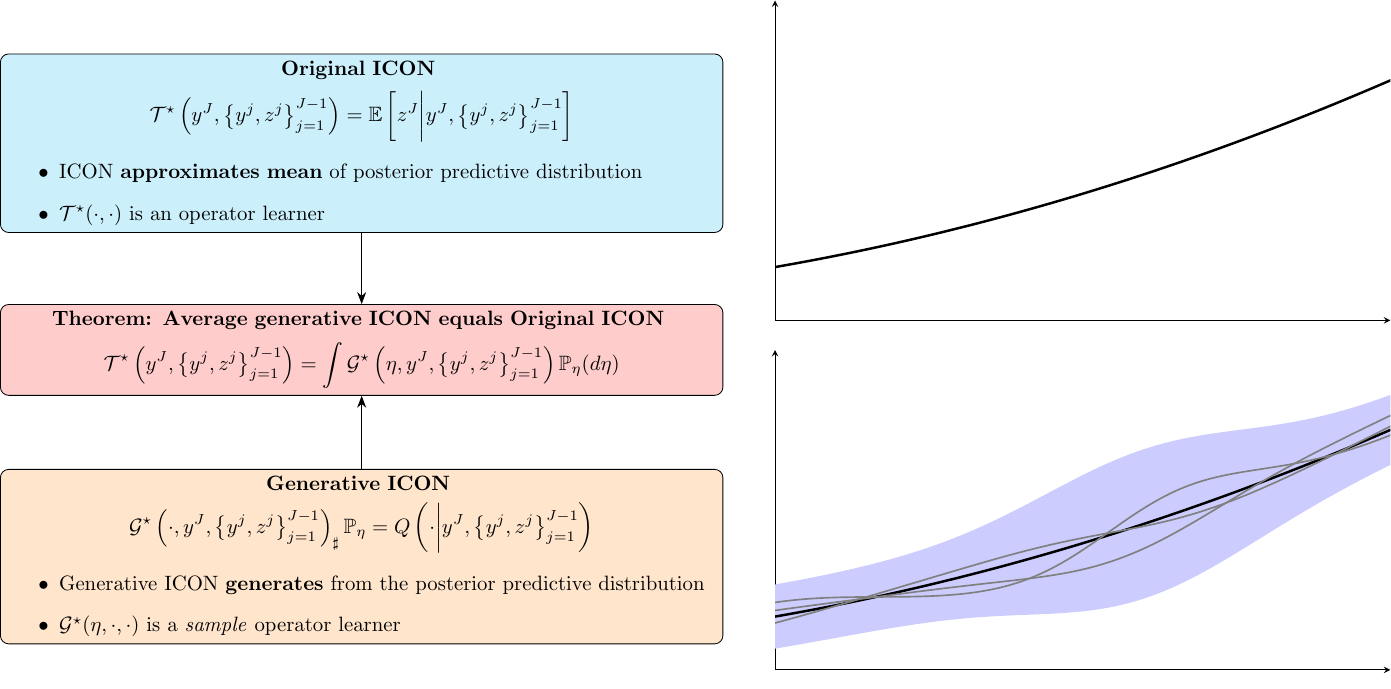}
    \caption{Illustration of the capabilities of ICON vs generative ICON. Standard ICON returns a single solution, which represents the approximation of the mean of the Bayesian posterior predictive distribution. Generative ICON characterizes uncertainty of the mean solution through samples from the posterior distribution, which is represented by the gray paths. The blue region represents the variance of the posterior predictive around the mean solution.    }
    \label{fig:icongen}
\end{figure*}

\begin{remark}[When should one use GenICON over ICON?]
If we are only interested in the first moment of \eqref{eq:probabilistic_ICON},
i.e.  the conditional expectation \eqref{eq:conditional_expectation_ICON},
then Theorem~\ref{thm:ICON_conditional} shows we do not need to consider a full probabilistic formulation, and the least squares regression \eqref{eq:icontraining} is sufficient. If however, we need probabilistic information, such as uncertainty quantification bounds for predictive guarantees for the conditional expectation predictions, then we need the full probabilistic formulation.
The full probabilistic formulation is necessary for numerous other critical tasks, e.g. formulating and solving inverse problems using ICON, or carrying out  model fine-tuning with UQ capabilities.
\end{remark}

\subsection{A generative ICON model based on generative adversarial networks}  
\label{sec:icon_gan}

We detail one approach for constructing a generative model for sampling from the posterior predictive distribution \eqref{eq:iconbayes}. Suppose we are given the training data in the form of \eqref{eq:icontraining}. To simplify presentation, denote the context, i.e., example condition-QoI pairs, by $\mathbf{c} = \{y^j,z^j \}_{j = 1}^{J-1}$. Let $\mathcal{G}_\theta: H \times  (Y\times Z)^{J-1} \to Z$ be a family of conditional generators parametrized by $\theta$ that maps reference samples $\eta \sim \mathbb{P}_\eta$ and context $\mathbf{c}$ to an output sample $z^J_\theta = \mathcal{G}_\theta(\eta, \mathbf{c})$. The map induces a conditional measure $\mathbb{Q}_{ z^J_\theta| \mathbf{c}; \theta}(\cdot|\mathbf{c}) = \mathcal{G}_\theta(\cdot,\mathbf{c})_\sharp \mathbb{P}_\eta$, as well as a parametrized family of joint measures $\mathbb{Q}_{\mathbf{c}, z^J_\theta; \theta}(\cdot) = \mathbb{Q}_{ z^J_\theta| \mathbf{c}; \theta}(\cdot|\mathbf{c})\otimes \mathbb{Q}_{\mathbf{c}}(\cdot)$, which will approximate the joint measure $\mathbb{Q}^J(\cdot)$ in \eqref{eq:iconjointmeasure}. 

To find the best approximation, we minimize a probability divergence from $\mathbb{Q}_{\mathbf{c},z_\theta^J;\theta}$ to $\mathbb{Q}^J$; we choose to minimize the $(f,\Gamma)$-divergence which interpolates between $f$-divergences (such as KL divergences) and integral probability metrics (such as the Wasserstein-1 distance) \cite{birrell2022f}. These divergences are able to stably compare distributions that are mutually singular and avoid mode collapse when training GANs. The function $f: \R \to [0,\infty)$ is a convex univariate function and $\Gamma$ is a function space. The $(f,\Gamma)$-divergence is defined through a variational representation, which is also more amenable to computation. For notational convenience, let $\mathbb{Q}_\theta = \mathbb{Q}_{\mathbf{c}, z^J_\theta; \theta}$, the $(f,\Gamma)$-divergence from $\bbQ_\theta$ to $\bbQ^J$ is
\begin{align}
    \mathfrak{D}^\Gamma_f(\mathbb{Q}^J \| \mathbb{Q}_{\theta}) = \sup_{\mathcal{D} \in \Gamma} \left\{\Ex_{\mathbb{Q}^J}\left[\mathcal{D}(\mathbf{c},z^J)\right] - \mathbb{E}_{\mathbb{Q}_{\theta}}\left[ (f^\star \circ \mathcal{D})(\mathbf{c},z^J) \right] \right\}, 
\end{align}
where $f^\star$ is the Fenchel dual of $f$, i.e., $f^\star(p) = \sup_{x \in \R}\left[xp - f(x) \right]$. In the context of generative adversarial networks (GANs), the function $\mathcal{D}$ is interpreted as a discriminator. Since the random variables we consider are Banach space valued, $\mathcal{D}$ is an operator between Banach spaces and $\Gamma$ is a space of operators. In particular, here we have $\mathcal{D}: (Y\times Z)^J \to \R$. 

Notice that both the joint distribution $\mathbb{Q}^J(\cdot) = \mathbb{Q}_{z^J | \mathbf{c}}(\cdot|\mathbf{c})\otimes\mathbb{Q}_{\mathbf{c}}(\cdot)$ and the approximating family $\mathbb{Q}_{\mathbf{c},z_\theta^J;\theta}(\cdot) = \mathbb{Q}_{\mathbf{c},z_\theta^J;\theta}(\cdot | \mathbf{c})\otimes \mathbb{Q}_{\mathbf{c}}(\cdot)$  have the same marginal $\mathbb{Q}_\mathbf{c}$ so that we may write the divergence as
\begin{align*}
    \mathfrak{D}^\Gamma_f(\mathbb{Q}^J \| \mathbb{Q}_\theta) = \sup_{\mathcal{D} \in \Gamma}\left\{\Ex_{\mathbb{Q}_\mathbf{c}}\left[ \Ex_{\mathbb{Q}_{z^J|\mathbf{c}}}\left[\mathcal{D}(\mathbf{c}, z^J)\right] - \Ex_{\mathbb{Q}_\theta(z^J|\mathbf{c})}\left[ (f^\star \circ \mathcal{D})(\mathbf{c}, z^J) \right] \right] \right\}.
\end{align*}
To train the generator $\mathcal{G}_\theta$, we parametrize the discriminator by some family parametrized by $\phi$ and we have
\begin{align}
    \min_\theta\mathfrak{D}^\Gamma_f(\mathbb{Q}^J \| \mathbb{Q}_\theta) = \min_\theta \max_{\phi}\left\{\Ex_{\mathbb{Q}_\mathbf{c}}\left[ \Ex_{\mathbb{Q}_{z^J|\mathbf{c}}}\left[\mathcal{D}_\phi(\mathbf{c}, z^J)\right] - \Ex_{\mathbb{P}_\eta}\left[ (f^\star \circ \mathcal{D}_\phi)(\mathbf{c},\mathcal{G}_\theta\left(\eta, \mathbf{c}\right)) \right] \right] \right\}.
\end{align}
Given training samples of the form \eqref{eq:icontraining}, $\left\{\{(y^j_m,z^j_m) \}_{j = 1}^{J}\right\}_{m = 1}^M = \{ (\mathbf{c}_m, z^J_m)\}_{m = 1}^M$, the minimax problem is approximated as 
\begin{align}
    \min_\theta\max_\phi\left\{\frac{1}{M}\sum_{m = 1}^M\mathcal{D}_\phi(\mathbf{c}_m, z^J_m) - \frac{1}{M}\sum_{m = 1}^M (f^\star \circ \mathcal{D}_\phi)(\mathbf{c}_m,\mathcal{G}_\theta\left(\eta_m, \mathbf{c}_m\right)) \right\} \label{eq:genicon_objective},
\end{align}
where $\eta_m \sim \mathbb{P}_\eta$. 

In our experiments, we choose $f(x) = x\log x$ ($f^\star(p) = e^{p-1}$), which corresponds with the forward KL divergence, and $\Gamma$ to be the space of $L$-Lipschitz operators. In practice, the Lipschitz constraint stabilizes the training of GANs \cite{arjovsky2017wasserstein,birrell2022f} and is implemented through a one-sided penalty \cite{birrell2022f}. The resulting divergence is known as the Lipschitz-regularized KL divergence, which is capable of comparing distributions that are mutually singular. The full training objective is as follows
\begin{align}
    \min_\theta \mathfrak{D}_{\mathrm{KL}}^\Gamma( \bbQ^J \,\|\, \bbQ_\theta) \approx 
    \min_\theta \max_{\phi } \Bigg\{
    &\frac{1}{M} \sum_{m = 1}^M \mathcal{D}_\phi(\bfc_m, z^J_m) 
    - \frac{1}{M} \sum_{m = 1}^M \exp\Big( \mathcal{D}_\phi(\bfc_m, \mathcal{G}_\theta(\eta_m, \bfc_m)) - 1 \Big) \notag \\
    &- \lambda \cdot \frac{1}{M} \sum_{m = 1}^M \max\left(0, \frac{\|\nabla_{z^J} \mathcal{D}_\phi({\bfc}_m, \tilde{z}^J_m)\|^2}{L^2} - 1 \right)
    \Bigg\},
\end{align}
where:
\begin{itemize}
    \item $(\tilde{\bfc}_m, \tilde{z}^J_{m}) \sim \rho$ are samples used to estimate the gradient penalty, typically obtained by interpolating between real and generated samples;
    \item $\lambda > 0$ is a hyperparameter controlling the penalty strength; and
    \item $L > 0$ is the target Lipschitz constant.
\end{itemize}

The components of the model are implemented as follows: 
\begin{itemize}
\item The \textbf{generator} $\mathcal{G}_\theta(\eta, \bfc)$ is a transformer-based, decoder-only neural network that maps reference samples $\eta \sim \mathbb{P}_\eta$ and input context $\bfc$ to predicted samples $z^J_\theta$, aiming to approximate the posterior predictive distribution $\bbQ_{z^J_\theta|\bfc ;\theta}(\cdot|\bfc)$;
\item The \textbf{discriminator} $\mathcal{D}_\phi(\bfc, z^J)$ is a transformer-based, decoder-only neural network that takes input context $\bfc$ and a predicted value $z^J$ as input, and outputs a scalar score for each $z^J$; and

    \item The \textbf{gradient penalty} is a Lipschitz regularization term,
    \[
        \max\left\{0, \frac{\|\nabla_{z^J} \mathcal{D}_\phi(\tilde{\bfc}, \tilde{z}^J)\|^2}{L^2} - 1\right\},
    \]
    is computed on interpolated samples $(\tilde{\bfc}, \tilde{z}^J)$. This term only penalizes the discriminator when the local gradient norm exceeds $L$, instead of strictly enforcing an $L$-Lipschitz constraint. As a result, $\mathcal{D}_\phi$ is encouraged but not required to satisfy $\|\nabla_\bfZ \mathcal{D}_\phi\| \leq L$ everywhere. (See Section 4 of \cite{birrell2022f}.)
\end{itemize}
Overall, this training objective enforces approximate Lipschitz continuity of the discriminator and enables stable adversarial training. The objective promotes both fidelity to the target conditional distribution and regularity of the discriminator, thus preserving the theoretical properties of divergence-based GAN training.

\section{Numerical examples}
\label{sec:numerical}

We demonstrate the Bayesian interpretation of ICON as well as our novel generative ICON through some simple numerical examples. Sections~\ref{sec:nonidentifiable} and~\ref{sec:partialobs} demonstrate that ICON is robust to degenerate posteriors, model non-identifiability, and partial and noisy observations. These properties can be easily explained through a Bayesian inference. Sections~\ref{sec:genICONnoisyobs} and~\ref{sec:genICONrandombvp} demonstrate generative ICON and its ability to capture the posterior predictive distribution. Examples that use ICON use the same model architecture as described in \cite{yang2023context}, while the generative ICON model architecture and training is described in Appendix~\ref{sec:nn_structure}.

\subsection{Robustness of ICON to degenerate posteriors and non-identifiability}
\label{sec:nonidentifiable}

We consider variations of the forward and inverse problem for Example~\ref{ex:ode} through the Bayesian interpretation of ICON. We consider the same random ODE under two different parametrizations: 
\begin{align}
     \begin{cases}
    \frac{d}{dt}u(t,\omega) = \gamma_1(\omega)c(t,\omega) u(t,\omega) + \gamma_2(\omega), \\
    u(0,\omega) = u_0(\omega), 
    \end{cases} \begin{cases}
    \frac{d}{dt}u(t,\omega) = \eta_1(\omega) \eta_3(\omega)c(t,\omega) u(t,\omega) + \eta_2(\omega), \\
    u(0,\omega) = u_0(\omega).
    \end{cases}
\end{align}
We let the parameters be the coefficients: $\alpha(\omega) = (\gamma_1(\omega),\gamma_2(\omega)) \in \R^2$ and $\alpha(\omega) = (\eta_1(\omega),\eta_2(\omega),\eta_3(\omega)) \in \R^3$. In the forward problem, the conditions are $y(\omega) = (c(t,\omega), u_0(\omega) \in L^2(0,T) \times \R$ and the quantity of interest be $z(\omega) = u(t,\omega) \in L^2(0,T)$. In the inverse problem, the roles are reversed --- $y(\omega) = u(t,\omega) \in L^2(0,T)$ and $z(t,\omega) = c(t,\omega) \in L^2(0,T)$. For all cases, we let $T = 1$. Notice that in the 2-parameter problem, the parameters are exactly identifiable with only a single demo (example-QoI pair). Given $(c(t,\omega), u(t,\omega)$, the parameters $(\gamma_1,\gamma_2)$ are determined by solving the linear system
\begin{align}\label{eq:parameteridentify}
    \begin{bmatrix}
        u'(0) \\ u'(1)
    \end{bmatrix} = \begin{bmatrix}
        c(0)u(0) & 1 \\ c(1) u(1) & 1 
    \end{bmatrix} \begin{bmatrix}
        \gamma_1 \\ \gamma_2
    \end{bmatrix}.
\end{align}
On the other hand, notice that the 3-parameter case is \emph{not identifiable} as any combination of $\eta_1,\eta_3$ such that $\eta_1\eta_3 = \gamma_1^\star$ can result in the same differential equation no matter how many more examples are given. We first explore ICON's robustness to this type of non-identifiability and then study how Bayesian inference can explain this robustness.

We generate two synthetic datasets from the linear ODE in Equation~\eqref{eq:ex1}. The \textit{2-parameter (identifiable)} dataset uses
$\frac{d}{dt}u(t,\omega) = \gamma_1(\omega)\,c(t,\omega)\,u(t,\omega) + \gamma_2(\omega),$
while the \textit{3-parameter (non-identifiable)} dataset uses
$\frac{d}{dt}u(t,\omega) = \eta_1(\omega)\,\eta_3(\omega)\,c(t,\omega)\,u(t,\omega) + \eta_2(\omega).$ For the forward problem, we condition on the initial state \(u_0(\omega)\) and the input \(c(t,\omega)\), and predict the trajectory \(u(t,\omega)\).  For the  inverse problem, we invert this: conditioning on \(u(t,\omega)\) to recover \(c(t,\omega)\).  ICON is trained separately on each dataset by minimizing the loss defined in \eqref{eq:icontraining}, and both tasks are then evaluated on each dataset.
Table~\ref{tab:relative-errors-sci} reports the \emph{mean} relative error for each train$\to$test combination.  In all cases, forward and inverse errors are below \(2\times10^{-3}\).  

\begin{table}[ht]
  \centering
  \caption{Relative errors for forward and inverse problems of an ODE problem}
  \label{tab:relative-errors-sci}
  \begin{tabular}{lcccc}
    \toprule
     & \multicolumn{2}{c}{Forward Problem} & \multicolumn{2}{c}{Inverse Problem} \\
    Train $\rightarrow$ Test $\downarrow$
      & 3-parameter 
      & 2-parameter 
      & 3-parameter  
      & 2-parameter  \\
    \midrule
    Train on 2-parameter (identifiable)
      & $5.66\times10^{-4}$ & $9.00\times10^{-4}$
      & $1.55\times10^{-3}$ & $1.19\times10^{-3}$ \\
    Train on 3-parameter (non-identifiable)
      & $8.19\times10^{-4}$ & $1.51\times10^{-3}$
      & $9.10\times10^{-4}$ & $9.44\times10^{-4}$ \\
    \bottomrule
  \end{tabular}
\end{table}

\paragraph{Bayesian inference explains why ICON is robust to non-identifiability and degenerate posteriors.}
We characterize the Bayesian posterior and posterior predictive for the two problems. Let $\mathcal{F}: L^2(0,T) \times L^2(0,T) \to \R^2$ denote the linear mapping from $(c(t,\omega),u(t,\omega))$ to the true parameters $(\gamma_1^\star,\gamma_2^\star)$; this mapping is the solution of \eqref{eq:parameteridentify}. Similarly, let $\mathcal{L}: L^2(0,T) \times \R \times \R^2$ be the solution operator to the linear ODE that takes $(c(t,\omega),u_0(\omega),\gamma_1(\omega), \gamma_2(\omega))$ to $u(t,\omega)$. The posterior distribution for the identifiable problem is a delta distribution centered at the true value $\gamma_1^\star, \gamma_2^\star$. Assuming the support of the prior on $\gamma_1,\gamma_2$ contains the truth, we have
\begin{align} \label{eq:2parampost}
   \mathbb{Q}_{\gamma_1,\gamma_2 \mid c,u}\!\left(\,\cdot \,\middle|\, c(t,\omega), u(t,\omega)\right) 
&\propto \mathbb{P}_{u \mid c,\gamma_1,\gamma_2}\!\left(\,\cdot \,\middle|\, c(t,\omega), \gamma_1(\omega), \gamma_2(\omega)\right)\;
   \mathbb{Q}_{\gamma_1,\gamma_2}(\cdot) \nonumber \\
&= \delta_{\mathcal{F}(u(t,\omega),c(t,\omega))}(\cdot)\,\mathbb{P}_{\gamma_1,\gamma_2}(\cdot) \nonumber \\
&= \delta_{(\gamma_1^\star,\gamma_2^\star)}(\cdot),
\end{align}
where $\delta_x$ is a Dirac measure centered at $x$. Meanwhile, the non-identifiable problem will have a posterior distribution that is also degenerate but supported on a manifold of $(\eta_1,\eta_2,\eta_3) \in \R^3$ in which $\eta_1\eta_3 = \gamma_1^\star$. Again, assuming the support of the prior on $\eta_1, \eta_3$ contains the truth, we have
\begin{align} \label{eq:3parampost}
\mathbb{Q}_{\eta_1,\eta_2,\eta_3 \mid c,u}\!\left(\,\cdot \,\middle|\, c(t,\omega), u(t,\omega)\right) 
&\propto \mathbb{P}_{u \mid c,\eta_1,\eta_2,\eta_3}\!\left(\,\cdot \,\middle|\, c(t,\omega), \eta_1(\omega), \eta_2(\omega), \eta_3(\omega)\right)\;
   \mathbb{Q}_{\eta_1,\eta_2,\eta_3}(\cdot) \nonumber \\
&= \delta_{\{\,(\eta_1,\eta_2,\eta_3): \; \eta_1 \eta_3 = \gamma_1^\star,\; \eta_2 = \gamma_2^\star \,\}}(\cdot)\,\mathbb{P}_{\eta_1,\eta_2,\eta_3}(\cdot).
\end{align}
The probability mass of the posterior distribution will shift according to the prior distribution on $(\eta_1,\eta_2,\eta_3)$. The posterior predictive distribution is computed as follows
\begin{align}
\mathbb{Q}_{\tilde{u}\mid \tilde{c},\tilde{u}_0,(c,u)}(\cdot | \tilde{c}, \tilde{u}_0, c,u)
&= \int \mathbb{P}_{\tilde{u} \mid \tilde{c}, \tilde{u}_0, \gamma_1,\gamma_2}(\,\cdot \mid \tilde{c},\tilde{u}_0, \gamma_1, \gamma_2)\; \nonumber
   \mathbb{Q}_{\gamma_1,\gamma_2 \mid c,u}(d\gamma_1,d\gamma_2 \mid c,u) \\
&= \int \delta_{\,\mathcal{L}(\tilde{c},\tilde{u}_0,\gamma_1,\gamma_2)}(\cdot)\;
   \delta_{(\gamma_1^\star,\gamma_2^\star)}(d\gamma_1,d\gamma_2) \nonumber  \\
&= \delta_{\,\mathcal{L}(\tilde{c},\tilde{u}_0,\gamma_1^\star,\gamma_2^\star)}(\cdot).
\end{align}
A similar computation with the parameters $(\eta_1,\eta_2,\eta_3)$ will result in the same posterior predictive. We can see this easily by observing \eqref{eq:postpred_differentlatents} which states the same posterior predictive can arise from a variety of latent variables. Moreover, the same computation can be done to characterize the inverse problem --- the distribution of $\tilde{c}$ given $\tilde{u}, (c,u)$. In classical Bayesian inference, sampling from such a posterior distribution poses a challenge as sampling needs to occur \emph{exactly} along the low dimensional manifold. ICON directly approximates the mean of the posterior predictive which, in this case, is delta function, meaning that the mean matches the entire distribution, thereby bypassing this sampling challenge. Our numerical experiment confirms that ICON’s direct approximation of the mean of the posterior predictive --- and not the full posterior --- bypasses the bottleneck of computing the posterior.

\paragraph{Bayesian perspective informs the value of $J$.}

The exact characterization of the full Bayesian solution shows that there is no uncertainty in the posterior predictive when the likelihood is given exactly. ICON, however, learns the ODE and its operator in a likelihood-free manner, i.e., without knowledge of the exact form of the ODE system.  The exact characterization of the posterior and posterior predictive informs ICON should be able to approximate the conditional mean of the posterior predictive with only a single example-QoI pair. That is, we should choose $J = 1$. 

We train ICON on a dataset consisting solely of the ODE example~\ref{ex:ode}, using a standard regression training scheme. During training, the model is presented with in-context learning tasks in which the number of demo examples varies from 1 to 5, and each demo consists of 40–50 input-output pairs. After training, ICON is evaluated on its ability to perform in-context prediction for new ODE problems drawn from the same distribution. As shown in Figure~\ref{fig:ode_1demo}, the model successfully learns to perform operator regression in context. In particular, even when given only a single demo with as few as 3 points, ICON is able to produce accurate predictions for the QoI. The prediction quality improves with longer demo sequences, but the results demonstrate that ICON requires very little data to capture the underlying ODE.

\begin{figure}[h]
    \centering
    \begin{minipage}[b]{0.49\linewidth}
        \centering
        \includegraphics[width=\linewidth]{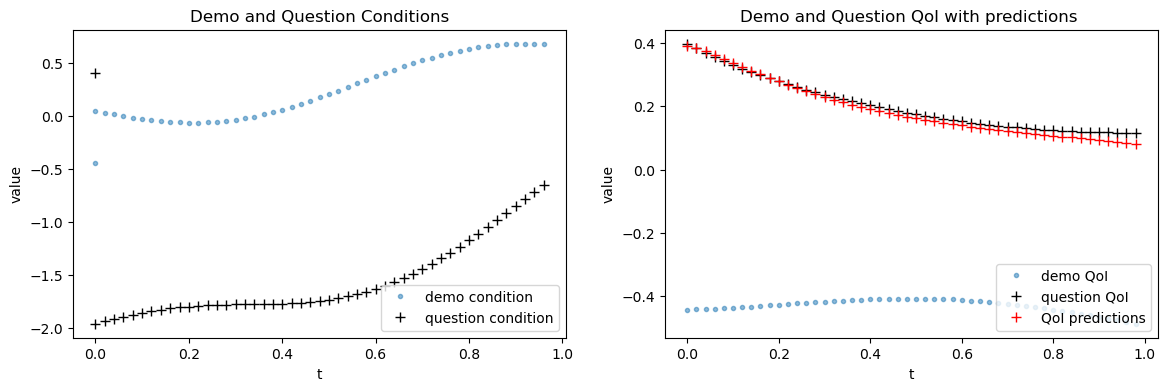}
        \caption*{(a) Demo of length $50$} 
    \end{minipage}
    \hfill
    \begin{minipage}[b]{0.49\linewidth}
        \centering
        \includegraphics[width=\linewidth]{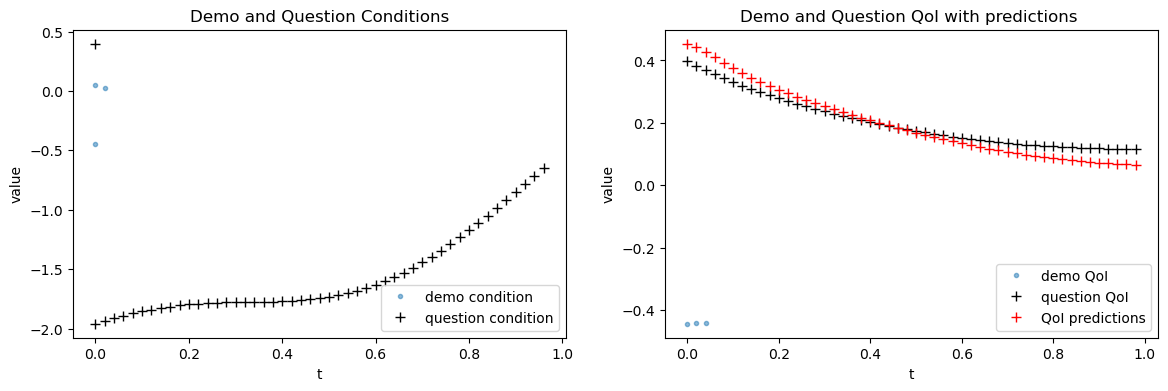}
        \caption*{(b) Demo of length $3$} 
    \end{minipage}
    \caption{ICON is able to recover the correct ODE behavior using only a single example condition-QoI pair. We consider the forward problem $u'(t) = 1.04696604\, c(t)\, u(t) + 0.07126787$, where the condition is $(c(t), u(0))$ and the QoI is $u(t)$ for $t \in (0,1)$. (a) Prediction using a demo with $50$ points and a question condition of $50$ points. (b) Prediction using a demo with only $3$ points (i.e., values of $u(t)$ at $t = 0, 0.02, 0.04$ and values of $c(t)$ at $t = 0, 0.02$), with the same $50$-point question condition. Despite the extremely short demo in (b), ICON is able to accurately predict the QoI, and performance improves with longer demos as seen in (a).}

    \label{fig:ode_1demo}
\end{figure}

\subsection{Robustness of ICON to partial and noisy observations}
\label{sec:partialobs}

We apply ICON to a variation of the reaction-diffusion problem (Example~\ref{ex:bvprd}) to demonstrate its robustness to partial and noisy observations. This robustness can be explained simply through the Bayesian perspective. On $x \in [0,1]$ we solve
\begin{align}
    \begin{cases}
-0.05\,a(\omega)\,u''(x,\omega) \;+\; k(x,\omega)\,u(x,\omega)
= c(\omega), \\
u(0,\omega)=u_0(\omega),\;u(1,\omega)=u_1(\omega).
\end{cases}
\end{align}
We let the function $\alpha(\omega) = k(x,\omega) \in L^2(0,1)$ be the parameter, let the boundary conditions, source term and parameter $a$ be the conditions $y(\omega) = (u_0(\omega), u_1(\omega), a(\omega), c(\omega)) \in \R^4$, and the quantity of interest be noisy observations of the solution at $M$ points on the domain, where $M$ is less than the total number of discretization points of the domain. That is, $ z(\omega) = \mathbf{u}(\omega) + \epsilon(\omega) \in \R^M$, where $\mathbf{u}(\omega)_i = u(x_i,\omega)$, and $\epsilon \sim \mathcal{N}(0,\sigma^2\mathbf{I})$. Given new conditions and noisy example-QoI pairs corresponding to the same parameter field $k(x,\omega)$, we will use ICON to predict the new solution. The Bayesian interpretation of ICON tells us that ICON needs to implicitly invert for the the parameter field $k(x,\omega)$ from the example pairs, which determines the solution operator, and apply that operator to new conditions.

The parameter field $k(x,\omega)$ is modeled as a softmax-transformed Gaussian process. To generate the training data, we sample $1000$ source terms $k(x,\omega)$; for each realization of $k(x,\omega)$ we generate $100$ samples of the source, parameter, and boundary conditions, by $(u_0,u_1) \sim \mathcal{U}([-1,1]^2)$, $a \sim \mathcal{U}[0.5,1.5]$, and $c \sim \mathcal{U}[-2,2]$. The PDEs are solved numerically on a uniform grid with spacing $\Delta x = 0.01$ (meaning the solution is approximated on $100$ grid points). The QoI is constructed by sampling the solution $u(x)$ at $M$ points, $M \in \{40,50 \}$, and then adding independent Gaussian noise with $\sigma = 0.1$ at each point.

We train ICON using a standard regression training scheme, i.e.  minimizing the loss function~\ref{eq:icontraining} by providing \emph{demo} pairs of the form $(y_m,z_m) = (u_{0m}, u_{1m}, a_m, c_m, \mathbf{u}_m + \epsilon_m)$, where $\epsilon_m \sim \mathcal{N}(0,0.1^2\mathbf{I})$. Notably, the data is corrupted by noise \emph{and} only contains partial observations of the solutions. Moreover, the partial observations are \emph{heterogeneous}, meaning that different demos have observations of the solution at different points in the domain. The demos pairs remain corrupted by noise when evaluating the model after training. We will see that ICON is able to return \emph{de-noised} solutions. 

Figure~\ref{fig:pde_partial_obs_example} shows a representative prediction: although ICON is never exposed to clean solution data during training, its prediction (red $+$) aligns closely with the true noise-free solution (black $+$). Figure~\ref{fig:pde_partial_obs_errors} reports the average relative error $\mathbb{E}\left[\|u_{\mathrm{pred}} - u_{ground\, truth}\|/\|u_{ground\, truth}\|\right]$. The error consistently stays below $2\%$ and decreases as more demonstrations are provided. These results highlight ICON’s robustness to noise and its ability to extract relevant structure from heterogeneous, partially observed data. The ability of ICON to return denoised predictions is clear with its conditional expectation characterization in Theorem~\ref{thm:ICON_conditional} --- with mean-zero additive noise, we have that 
\begin{align*}
    \Ex\left[z(x,\omega) | y^J, \{y^j,z^j \}_{j = 1}^{J-1} \right] &= \Ex\left[u(x,\omega) + \epsilon(\omega) | y^J, \{y^j,z^j \}_{j = 1}^{J-1}  \right]\\ & = \Ex[u(x,\omega) | y^J, \{y^j,z^j \}_{j = 1}^{J-1}].
\end{align*}
Unlike deterministic operator learning approaches, which require full and clean data to make reliable predictions, ICON's Bayesian perspective allows it to effectively learn from noisy and partial observations. It also shows ICON's ability to implicitly infer the high-dimensional parameter field $k(x,\omega)$ from limited and noisy context and demonstrates ICON's ability to do inference in challenging and more realistic PDE settings.

\medskip

\begin{figure}[h]
  \centering
  \begin{minipage}{0.49\linewidth}
    \centering
    \includegraphics[width=\linewidth]{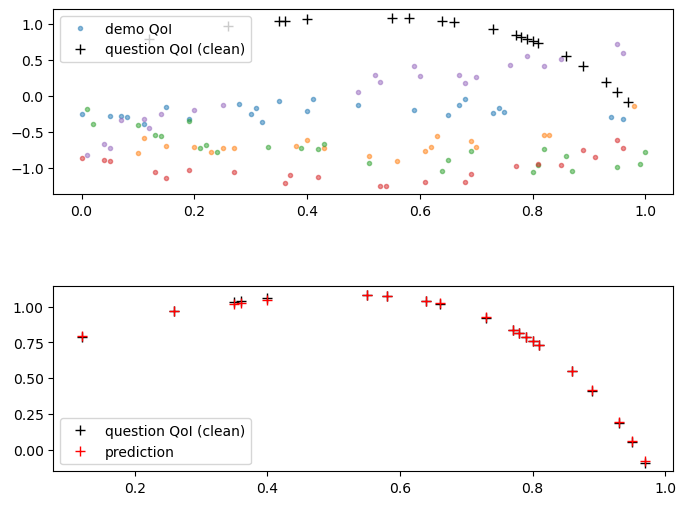}
    \caption{Visualization of in-context operator learning of the reaction–diffusion task with noisy QoIs. The model predicts QoIs aligning well with noise-free QoIs.}
    \label{fig:pde_partial_obs_example}
  \end{minipage}\hfill
  \begin{minipage}{0.49\linewidth}
    \centering
    \includegraphics[width=\linewidth]{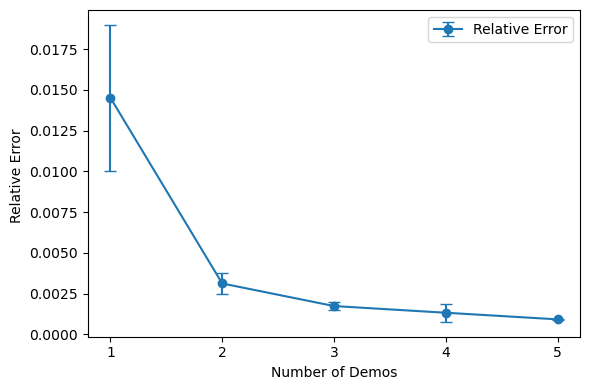}
    \caption{Average relative testing errors for the reaction–diffusion task versus demos per prompt.}
    \label{fig:pde_partial_obs_errors}
  \end{minipage}
\end{figure}

\subsection{Generative ICON on elliptic PDEs with noisy observations}
\label{sec:genICONnoisyobs}
We study a one-dimensional random Poisson equation on the domain \( x \in [0,1] \),
\begin{align}
    \begin{cases}
        -  u''(x,\omega) = k(x,\omega), \\
        u(0,\omega) = u_0(\omega), \quad u(1,\omega) = u_1(\omega).
    \end{cases}
\end{align}
We designate the source terms $y(\omega) = k(x,\omega) \in L^2(0,1)$ to be the conditions and the boundary conditions $\alpha(\omega) = (u_0(\omega), u_1(\omega)) \in \R^2$ to be the parameters. The quantity-of-interest is the solution $u(x,\omega)$ with a simple additive noise: $z(\omega) = u(x,\omega) + \epsilon(\omega) \in L^2(0,1)$, $\epsilon \sim \mathcal{N}(0,\sigma^2)$. The noise is constant across the spatial domain for each realization $\omega$. 

To construct the training dataset, we first sample $100$ realizations of the source term $k(x,\omega)$, each drawn from a softmax transformed Gaussian. For each source term, we independently generated $10$ boundary value pairs $(u_0,u_1) \sim \mathcal{U}([-1,1]^2)$, yielding $1000$ unique elliptic PDE problems. These equations are numerically solved to obtain solution $u(x,\omega)$  to which we add an offset given by Gaussian noise to generate the observed QoI $z(\omega)$. We train a GenICON to model the posterior predictive distribution of noisy solutions $z(\omega)$ conditioned on prompt consisting of context (demonstration condition-QoI pairs) and a condition: 
$
y^j, \left\{ (y^{j}, z^{j}) \right\}_{j=1}^{J-1}.
$

We generate three datasets $\mathcal{D}_\sigma$ with varying noise levels, \( \sigma \in \{0.1, 0.05, 0.025\} \), and train GenICON separately on each dataset. For comparison, we also train a standard ICON regression model on \( Q_{0.1}(y, z) \), which outputs a single deterministic prediction for each query. The purpose of this experiment is to assess GenICON's ability to recover both the posterior predictive mean and variance. When trained on \( \mathcal{D}_\sigma\), GenICON should produce samples whose variance matches the prescribed noise level \( \sigma \). In contrast, ICON, which only recovers the mean, cannot capture uncertainty. Figure~\ref{fig:genicon_sigma_comparison} shows posterior predictive samples generated by GenICON as well as the mean and variance estimated through samples. Notice that the true solution is within the uncertainty bands of GenICON. We also show that ICON is unable to denoise the prediction with only $J -1 = 5$ demos and that it is unable to express its uncertainty in its prediction.  

To quantitatively evaluate performance, we apply the model on a test dataset. For each prompt (with $5$ demos and $25$ evenly spaced pts in condition-QoI), we perform $1000$ posterior samplings and compute the empirical variance across the predicted QoIs at each spatial point. The final metric is obtained by averaging over all query points and test prompts summarized in Table~\ref{tab:icon_noisy_variance}. Notice that the estimated noise level provides a decent estimate for the true noise level, even though there are only $J - 1 = 5$ demos. Moreover, the overestimate of the noise levels may be the fact that the generative models based on minimizing the (Lipschitz-regularized) forward KL divergence tends to be moment capturing, rather than mode capturing, meaning the generative model will tend to have high variance than the true distribution.

\begin{figure}[htbp]
    \centering
    \begin{subfigure}[t]{0.24\linewidth}
        \centering
        \includegraphics[width=\linewidth]{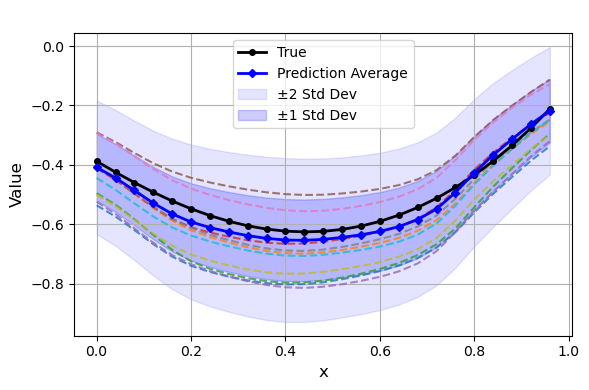}
        \caption{$\sigma = 0.1$}
        \label{fig:sigma01}
    \end{subfigure}
    \hfill
    \begin{subfigure}[t]{0.24\linewidth}
        \centering
        \includegraphics[width=\linewidth]{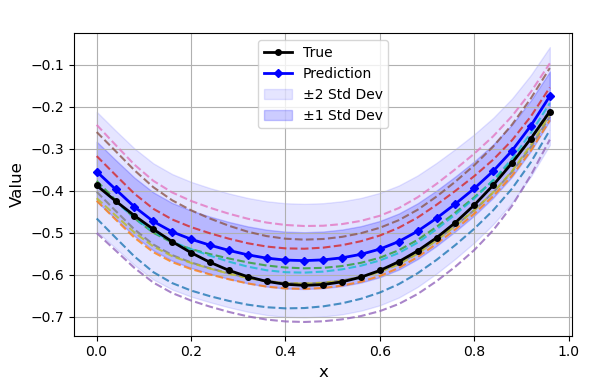}
        \caption{$\sigma = 0.05$}
        \label{fig:sigma005}
    \end{subfigure}
    \hfill
    \begin{subfigure}[t]{0.24\linewidth}
        \centering
        \includegraphics[width=\linewidth]{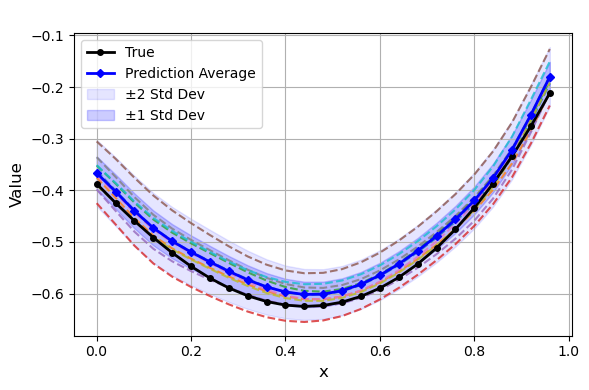}
        \caption{$\sigma = 0.025$}
        \label{fig:sigma0025}
    \end{subfigure}
    \hfill
    \begin{subfigure}[t]{0.24\linewidth}
        \centering
        \includegraphics[width=\linewidth]{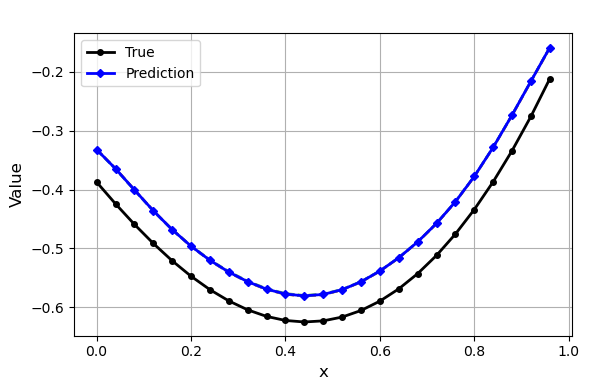}
        \caption{ICON regression, $\sigma = 0.1$}
        \label{fig:icon_reg}
    \end{subfigure}
    \caption{
    Posterior predictive samples generated by GenICON trained on \(\mathcal{D}_\sigma \) and evaluated on the corresponding test set with the same noise level \( \sigma \) for an elliptic differential equation operator with $u_0 = -0.38746364, u_1 = -0.14212306$. All predictions are conditioned on the same underlying differential equation, boundary conditions \( u_0(\omega), u_1(\omega) \) and same set of conditions-QoIs $\left(y^J, \{y^{j},z^{j} )\}_{j = 1}^{J-1}\right)$ for $J-1=5$ demos;  the  output variance varies in accordance with the training noise level $\sigma$. As \( \sigma \) decreases, the spread of the generated samples contracts accordingly, demonstrating GenICON's ability to calibrate uncertainty. In contrast, ICON regression returns a single deterministic solution regardless of noise.}
    \label{fig:genicon_sigma_comparison}
\end{figure}

\begin{table}[ht]
  \centering
  \footnotesize
  \caption{Prediction variance of GenICON (averaged across samples). The estimated noise level $\hat{\sigma}$ matches with the true noise levels $\sigma$. }
  \label{tab:icon_noisy_variance}
  \begin{tabular}{lcccc}
    \toprule
    Model \& dataset & GenICON \(\mathcal{D}_{0.1} \) & GenICON \( \mathcal{D}_{0.05}\) & GenICON \( \mathcal{D}_{0.025} \) & ICON \( \mathcal{D}_{0.1} \)\\
    \midrule
    Average Variance $\hat{\sigma}^2$
      & \(1.43 \times 10^{-2}\) 
      & \(3.81 \times 10^{-3}\) 
      & \(8.22 \times 10^{-4}\) 
      & -- \\
      {$\hat{\sigma}$}  & $0.119$ & $0.0617$ & $0.029$ & -- \\ 
    \bottomrule
  \end{tabular}
\end{table}

\subsection{Generative ICON on elliptic PDEs with random boundary conditions} 
\label{sec:genICONrandombvp}

We now consider a modified random elliptic problem to evaluate the robustness of GenICON under operator uncertainty due to random boundary conditions. Specifically, we study the Poisson equation on the domain $x \in [0,1]$,
\begin{align}
\begin{cases}
    - u''(x,\omega) = k(x,\omega), \, \\
    u(0,\omega) = u_0(\omega), u(1,\omega) = u_1(\omega),
    \end{cases}
\end{align}
where $u_0 \sim \mathcal{U}[-1,1]$ as before, but $u_1 | u_r \sim \mathcal{U}[u_r - 0.1, u_r + 0.1 ]$ and $u_r \sim \mathcal{U}[-1,1]$. The right boundary condition has a differently structured randomness than the previous example. There is no observational noise. In this example, the boundary conditions are the parameters, $\alpha(\omega) = (u_0,u_1) \in \R^2$, meaning that $(u_0,u_1)$ defines the operator instance. The source term is the condition, $y(\omega) = k(x,\omega) \in L^2(0,1)$, and the quantity-of-interest is the solution, $z(\omega) = u(x,\omega) \in L^2(0,1)$. 

The example condition-QoI pairs provide full trajectories $(k(x,\omega), u(x,\omega))$ which GenICON uses to infer the underlying operator-defining parameters $(u_0,u_1)$. In this example, we provide trajectories in which the left boundary condition is fixed, but the right boundary condition will be free, which means the inverse problem will be ill-posed and the solution remains non-deterministic as multiple operators corresponding with multiple right boundary conditions will be characterized. Therefore, a well-calibrated GenICON should output a \emph{distribution} of plausible solutions consistent with a fixed instance $k(x,\omega)$. 

To construct the dataset, we first sample $1000$ pairs of $(u_0,u_r) \sim \mathcal{U}([-1,1]^2)$. For each pair we generated $100$ problem instances by independently sampling a right boundary value $u_1 \sim \mathcal{U}(u_r-0.1, u_r+ 0.1)$, as well as a source term $k(x,\omega) \sim \text{softmax}(\mathcal{GP}(\mu(x), \sigma(x,x') ))$. Each resulting triple $(k(x), u_0,u_1)$ defines a distinct well-posed elliptic boundary value problem, which we solve numerically to obtain the solution $u(x)$. This process yields a dataset of $10^5$ labeled pairs $(k(x,\omega),u(x,\omega))$ which we use to train a GenICON. 

We assess GenICON's performance by observing the following three criteria. First, we examine whether the model can infer the correct operator-defining parameters from context (example condition-QoI pairs) with a fixed left boundary condition $u_0$ and fixed right boundary parameter $u_r$, as reflected by how closely the predicted QoI approximates the solution of the elliptic PDE given a new condition $k(x,\omega)$. In particular, we assess the recovery of the left boundary condition, which should be constant. Second, we evaluate GenICON's ability to generate a distribution of plausible solutions which should match with the posterior predictive distribution. Lastly, we evaluate the model's ability to capture the conditional structure of the right boundary $u_1$, which should vary between some fixed range $[u_r-0.1, u_r+0.1]$. A good GenICON should generate a distribution of plausible solutions rather than collapsing to a single point on the right boundary condition. 

We present our results in Figure~\ref{fig:genicon_poisson_freeur}. Figure~\ref{fig:poisson-free-ur-example} shows the posterior samples generated for a new source term and operator parameters $(u_0, u_r) = (-0.4764, 0.3894)$. All predicted solutions share the fixed left boundary condition while exhibiting variability at the right boundary condition. The sample mean at $x = 1$ closely aligns with the operator variability due to randomness at the right boundary condition. Figure~\ref{fig:poisson-free-ur-pt} quantifies the test error at both boundaries: predictions at $x = 0$ show low variance and error, indicating that GenICON accurately infers the left boundary conditions with low variance and error from even a single example condition-QoI pair. Predictions at $x = 1$ maintain high variance consistent with the training distribution, but gradually align with the expected mean as more examples are provided.

In Figure~\ref{fig:poisson-free-ur-pt} we report two evaluation metrics: physical consistency error and QoI prediction error. The physical consistency error quantifies how well the predicted solutions satisfy the governing elliptic PDE. Specifically, for a given $k(x,\omega)$ we predict QoI trajectory with a fixed right boundary condition. We then numerically solve an elliptic PDE with the same source term and boundary conditions and compare the discrepancy between the true solution and the predicted one. As the number of example condition-QoI pairs increase, GenICON gains more information about the underlying operator and the physical consistency error decreases accordingly. 

The QoI prediction error, on the other hand, measures how close the predicted solution is to a particular ground truth solution $\hat{u}(x)$ sampled from the data-generating distribution. Recall that in each in-context learning example, the prompt \(y^{J}, \{y^{j}, z^{j}\}_{j=1}^{J-1} \) is constructed from the same operator defined by the parameters $(u_0,u_r)$. However, due to the additional random structure in the boundary condition, there are many valid QoIs $z^J$ compatible with the same context. Therefore, the prediction error is inherently lower-bounded by the spread of the output distribution and cannot converge to zero. Instead the error saturates once the model captures the correct one-to-many mapping structure, representing the diversity of possible outputs.

\begin{figure}[htbp]
  \label{fig:genicon-poisson-free-ur}
    \centering
    \begin{subfigure}[t]{0.32\linewidth}
        \centering
        \includegraphics[width=\linewidth]{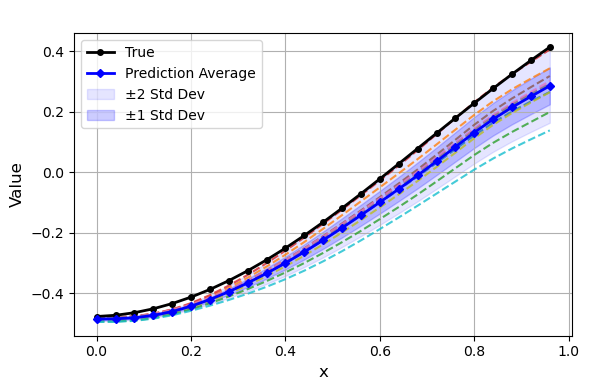}
        \caption{Example of QoI predictions}
        \label{fig:poisson-free-ur-example}
    \end{subfigure}
    \hfill
    \begin{subfigure}[t]{0.32\linewidth}
        \centering
        \includegraphics[width=\linewidth]{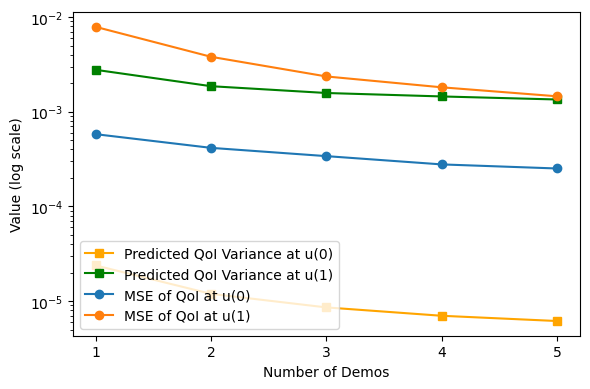}
        \caption{Test error at boundary points.}
        \label{fig:poisson-free-ur-pt}
    \end{subfigure}
    \hfill
    \begin{subfigure}[t]{0.32\linewidth}
        \centering
        \includegraphics[width=\linewidth]{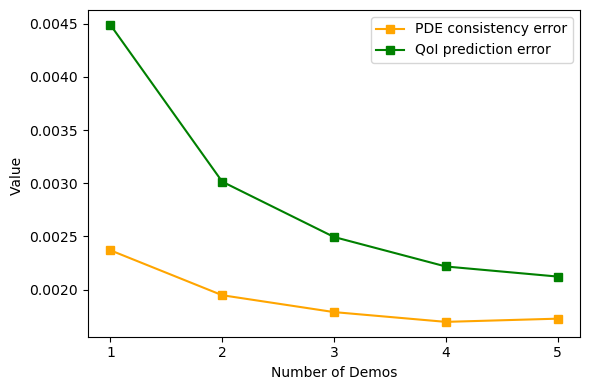}
        \caption{Test error of the interval}
        \label{fig:poisson-free-ur-stat}
    \end{subfigure}

    \caption{
GenICON predictions on elliptic PDEs with a random right boundary. (a) Predicted QoIs reflect boundary uncertainty: solutions agree at \( x = 0 \) but vary at \( x = 1 \). Dashed lines are sample trajectories which are used to compute the uncertainty bands. (b) Test errors at boundary points decrease with more demos; prediction at \( x = 1 \) remains diverse but aligns with the expected mean. (c) Physical consistency and QoI prediction errors improve with additional demos but saturate due to irreducible output variability.
}
    \label{fig:genicon_poisson_freeur}
\end{figure}

\section{Conclusion and outlook}

This paper develops a \emph{probabilistic} framework for operator learning and foundation models for ordinary and partial differential equations by formulating them in terms of \emph{random differential equations}.  Training data for operator learning then arises as sample realizations of random parameters, initial/boundary conditions, and solutions of RDEs. The resulting framework, \emph{probabilistic operator learning}, provides a unified framework for studying operator learning and foundation models. 

We apply the RDE formalism to interpret in-context operator learning as implicit Bayesian inference for differential equations. ICON can be understood as directly computing the posterior predictive mean of solutions given example condition-solutions, bypassing the need to explicitly characterize the posterior distribution. This makes ICON a form of likelihood-free, amortized inference. To our knowledge, this is the first formulation of in-context learning on Banach and Hilbert spaces, which are the natural settings for differential equations. The Bayesian perspective also explains ICON's robustness in inverse problems and under challenges such as non-identifiability, noisy data, partial observations. 

We further introduced a \emph{generative formulation of ICON} which extends ICON beyond approximating conditional expectations through characterizing the full posterior predictive distribution. We proved a general existence theorem for conditional generative models on infinite-dimensional spaces and showed that the expected GenICON operator recovers the deterministic ICON operator. 

Looking forward, we highlight several avenues for future investigation: 
\begin{itemize}
    \item \textbf{Errors and limitations of GenICON.} Our method assumes the GenICON model exactly characterizes the posterior predictive, but practical implementations incur approximation and finite-sample error. Quantifying these errors and assessing robustness of these generative models are important next steps for trustworthy application to scientific and engineering problems. Moreover, GenICON modeling demands significantly more data and computational resources than ICON, making it crucial to understand the sample complexity of GenICON.
    
    \item \textbf{\emph{Maximum a posteriori} estimates.} The posterior predictive mean may not always be the most informative statistic, especially when the PPD is multimodal --- for example, when multiple differential equations can explain the same context. Learning the posterior predictive \emph{mode} directly by, for example, training with mode-seeking divergences (e.g., reverse KL) or by maximizing the PPD through generative models may yield more informative statistics. 
    
    \item \textbf{Structure-preserving neural architectures. } GenICON architectures that respect PDE structure, such as Fourier neural operators, Galerkin-based neural architectures, and, given the RDE perspective, \emph{polynomial chaos}-based architectures \cite{xiu2002wiener,sharma2025polynomial,arampatzis2025generative} may be used to enhance the transformer architecture. 

    \item \textbf{Probabilistic formulations of other foundation models.} Our random differential equations formulation can also be applied to understand recent multi-operator learning methods such as POSEIDON \cite{herde2024poseidon} and PROSE \cite{liu2024prose}. Placing multi-operator learning methods in a probabilistic setting opens the door to systematic analyses of sample complexity and robustness. 
\end{itemize}
We hope that the probabilistic operator learning framework, together with the generative formulations of ICON, will contribute to advancing trustworthy scientific machine learning.

\bibliographystyle{unsrt}
\bibliography{pnas-sample}
\clearpage
\appendix

\section{Generative ICON architecture and training}\label{sec:nn_structure}

We adopt a GAN framework where both the generator and discriminator are decoder-only Transformers with identical architecture, summarized in Table~\ref{tab:transformer-config}. Each input token is projected into a 128-dimensional embedding space, combined with learned positional encodings, and processed by a 6-layer Transformer with causal masking. The generator additionally prepends a projected latent noise vector \( z \in \mathbb{R}^{8 \times 32} \) to the token sequence before Transformer processing, enabling sample diversity. The discriminator processes the same type of data sequence but without noise input. A final linear head maps the Transformer output to scalar predictions. Each model contains approximately 3.9 million parameters.

We utilize the AdamW optimizer with a warmup–cosine decay learning rate schedule to train the GAN model, using different momentum coefficients for the generator and discriminator optimizers, as detailed in Table~\ref{tab:optimizer-config}.

\begin{table}[h]
\centering
\caption{Transformer Configuration for Generator and Discriminator}
\label{tab:transformer-config}
\begin{tabular}{l c}
\toprule
\textbf{Component} & \textbf{Value} \\
\midrule
Number of Layers & 6 \\
Heads in Multi-Head Attention & 8 \\
Input/Output Dimension & 128 \\
Feedforward Hidden Dimension & 512 \\
Total Parameters (each) & 3.9M \\
\bottomrule
\end{tabular}
\end{table}

\begin{table}[h]
\centering
\caption{Optimizer and Learning Rate Schedule Configuration}
\label{tab:optimizer-config}
\begin{tabular}{lcc}
\toprule
\textbf{Component} & \textbf{Generator} & \textbf{Discriminator} \\
\midrule
Optimizer & AdamW &  AdamW \\
Peak Learning Rate & \multicolumn{2}{c}{$2 \times 10^{-5}$} \\
Learning Rate Schedule & \multicolumn{2}{c}{Warmup + Cosine Decay} \\
Warmup Percent & \multicolumn{2}{c}{1\%} \\
Gradient Norm Clipping & \multicolumn{2}{c}{10.0} \\
$\beta_1$ & 0.5 & 0.0 \\
$\beta_2$ & \multicolumn{2}{c}{0.9} \\
Gradient Penalty & -- & 0.1 \\
\midrule
Training Ratio (Disc:Gen) & \multicolumn{2}{c}{5:1} \\
Training Steps & \multicolumn{2}{c}{100{,}000} \\
Batch Size & \multicolumn{2}{c}{64} \\
\bottomrule
\end{tabular}
\end{table}

\end{document}